\documentclass[twoside]{article}

\usepackage[accepted]{aistats2026}
\usepackage[round,authoryear]{natbib}

\usepackage{tcolorbox}
\usepackage{xcolor} 
\usepackage{tikz}
\usepackage{todonotes}
\usepackage{amsmath}
\usepackage{amssymb}
\usepackage{mathtools}
\usepackage{amsthm}
\usepackage{thm-restate}
\usepackage{booktabs}
\usepackage{subcaption}

\newcommand{\fcircle}[2][]{\tikz \fill[#1] (0,0) circle (#2);}

\theoremstyle{plain}
\newtheorem{theorem}{Theorem}[section]

\newtheorem{lemma}[theorem]{Lemma}
\newtheorem{corollary}[theorem]{Corollary}
\theoremstyle{definition}

\theoremstyle{remark}
\newtheorem{remark}[theorem]{Remark}

\DeclareMathOperator*{\lin}{lin}
\DeclareMathOperator*{\ntk}{ntk}
\DeclareMathOperator*{\std}{std}
\DeclareMathOperator*{\diag}{diag}
\newcommand{\norm}[1]{\left\lVert#1\right\rVert}

\begin{document}

\twocolumn[

\aistatstitle{A Gaussian Process View on Observation Noise and Initialization in Wide Neural Networks}

\runningtitle{A Gaussian Process View on Observation Noise and Initialization in Wide Neural Networks}
\runningauthor{Calvo-Ordoñez et al.}

\vspace{-0.2in}

\aistatsauthor{
\hspace{0.4in}Sergio Calvo-Ordoñez\textsuperscript{*,1,2,3,\ensuremath{\dagger}} \And
\hspace{0.6in}Jonathan Plenk\textsuperscript{*,1,2} \And
\hspace{0.4in}Richard Bergna\textsuperscript{4} \And
\hspace{0.35in}Álvaro Cartea\textsuperscript{1,2}
\AND
\hspace{0.35in}Jose Miguel Hernández-Lobato\textsuperscript{4} \And
\hspace{0.35in}Konstantina Palla\textsuperscript{5} \And
\hspace{0.35in}Kamil Ciosek\textsuperscript{5}
}

\vspace{0.05in}

\aistatsaddress{
\small{\textsuperscript{1}Mathematical Institute, University of Oxford}\\
\small{\textsuperscript{2}Oxford-Man Institute, University of Oxford}\\
\small{\textsuperscript{3}OATML, University of Oxford}\\
\small{\textsuperscript{4}Department of Engineering, University of Cambridge}\\
\small{\textsuperscript{5}Spotify}\\
\small{\textsuperscript{*}Equal Contribution,}
\small{\textsuperscript{$\dagger$}Work partially done during a Spotify internship}\\
}

] 

\begin{abstract}
\vspace{-0.1in}
Performing gradient descent in a wide neural network is equivalent to computing the posterior mean of a Gaussian Process with the Neural Tangent Kernel (NTK-GP), for a specific prior mean and with zero observation noise. However, existing formulations have two limitations: (i) observation noise, since the NTK-GP assumes noiseless targets, leading to misspecification on noisy data; (ii) the equivalence does not extend to arbitrary prior means, which are essential for well-specified models. To address (i), we introduce a regularizer into the training objective, showing its correspondence to incorporating observation noise in the NTK-GP. To address (ii), we propose a \textit{shifted network} that enables arbitrary prior means and allows obtaining the posterior mean with gradient descent on a single network, without ensembling or kernel inversion. We validate our results with experiments across datasets and architectures, showing that this approach removes key obstacles to the practical use of NTK-GP equivalence in applied Gaussian process modeling.
\end{abstract}

\vspace{-0.2in}
\section{Introduction}
\vspace{-0.1in}

The connection between wide neural networks and Gaussian Processes (GPs) via the Neural Tangent Kernel (NTK) \citep{jacot2018neural} provides a powerful framework for understanding the training dynamics of deep learning. The NTK describes how predictions of a neural network evolve under gradient-based optimization in the infinite-width limit and has primarily been used to analyze generalization, convergence, and expressivity of neural networks from a theoretical perspective. While this line of work has yielded insights into training dynamics, its practical utility for GP inference remains limited. In particular, the standard NTK-GP correspondence, as established by \citet{lee2019wide}, connects the output of a wide neural network trained with gradient flow on mean squared error (MSE) to the posterior mean of a GP with the NTK as its kernel --- under restrictive conditions. Specifically, this result assumes (i) zero observation noise, leading to model misspecification when learning from real-world data, and (ii) a fixed prior mean given by a randomly initialized network, which is uninformative and not easily controllable. These constraints make the GP equivalence difficult to leverage in practice, as they limit the model’s ability to incorporate uncertainty and prior knowledge, two key concepts of principled Bayesian inference. In this work, we address these limitations and extend the theory in a way that enables practical posterior mean inference with NTK-GPs using a single neural network training run.

Concerning the first limitation, observation --- or aleatoric --- noise is a fundamental component of Gaussian Processes and other probabilistic models \citep{williams2006gaussian, saez2024neural, ordonezpartially} that represents observation uncertainty. Measurements in data are subject to precision limits, and labels can carry errors due to annotation noise or ambiguities \citep{kendall2017uncertainties}. By capturing this variability, aleatoric noise enables models to produce predictions that reflect the randomness of the observed process \citep{williams1998bayesian, bergna2024uncertainty}. In GPs, this is achieved through a variance term that accounts for label noise, contributing to robust and well-calibrated predictions.

\citet{hu2020simpleeffectiveregularizationmethods} tried to address the problem of observation noise by introducing a regularizer penalizing the distance between network parameters at time $t$, $\theta_t$, and their initialization, $\theta_0$, demonstrating its efficacy in noisy settings. Their analysis, however, relied on the NTK framework and assumed that the network remains in a linear regime throughout training --- a critical assumption that was left unproven. This oversight is significant, as the regularizer modifies the training dynamics and invalidates the application of previous results in \citet{lee2019wide}. Other recent works have used the regularizer proposed by \citet{hu2020simpleeffectiveregularizationmethods} to improve generalization and training stability. For instance, \citet{nitanda2020optimal} employed it to constrain networks near their initialization under the NTK regime, while \citet{suh2021non} extended this to deeper networks for recovering ground-truth functions from noisy data. \citet{he2020bayesian} explored the regularizer's role in Bayesian deep ensembles.

While these works demonstrate the utility of the regularizer, they all rely on the assumption that wide neural networks remain in a linear regime throughout training. This assumption is critical for leveraging the NTK to characterize the learning dynamics and derive theoretical guarantees, yet the validity of this assumption for the modified gradient flow has not been shown. This gap in the literature raises a key question: how does regularization affect the linearization of wide neural networks, and how does this connect to having a well-specified Bayesian model? In this work, we demonstrate that this regularizer not only preserves the linearity of neural networks but also introduces non-zero aleatoric noise into the NTK-GP mean posterior. We enable the NTK-GP posterior mean to reflect the noise present in real-world data, making it a properly specified Bayesian model. At the same time, this helps justify the generalization properties and benefits of regularization observed in previous literature.

To support inference with arbitrary prior means, we propose the use of a \textit{shifted network} during training. This approach provides a principled strategy to eliminate initialization randomness, ensuring deterministic convergence of a single \textit{shifted network} to the posterior mean of the defined NTK-GP prior.

Our contributions are summarized as follows:

\fcircle[fill=black]{2pt} We demonstrate that regularization in wide neural networks corresponds to aleatoric noise in NTK-GPs, correcting model misspecification.  

\fcircle[fill=black]{2pt} We prove that weight-space regularization ensures linear training dynamics under gradient flow and gradient descent.  

\fcircle[fill=black]{2pt} We introduce a shifted network framework, enabling arbitrary prior means and deterministic convergence to the NTK-GP posterior mean without ensembling or kernel inversion.  

\fcircle[fill=black]{2pt} We validate our theory with experiments on how architecture and regularization affect linearization and convergence.  

\section{Preliminaries}

\paragraph{Parametrizations: Standard versus NTK.} The choice of parametrization in neural networks determines how signals propagate and how gradients scale with the network width. Let $\phi:\mathbb{R}\to\mathbb{R}$ denote the activation function\footnote{We assume that $\phi$ and $\phi'$ are Lipschitz.}. In the standard parametrization, the layer outputs are defined (for $l=0,\ldots,L$) as:
\begin{equation}
h^{l+1} := W^{l+1}x^l + b^{l+1}, \quad x^{l+1} := \phi(h^{l+1}) \in \mathbb{R}^{n_{l+1}},
\end{equation}
where $W^{l+1} \in \mathbb{R}^{n_{l+1} \times n_{l}}$ and $b^{l+1} \in \mathbb{R}^{n_{l+1}}$ are the weights and biases, respectively. The parameter vector $\theta \in \mathbb{R}^p$ is defined by stacking them. Here, $n_0 := d$ is the dimension of the input $x^0=x$, and  $n_{L+1} := k$ is the dimension of the output $f(x,\theta) := h^{L+1}$. For simplicity we only consider $k=1$. The initial weights $\theta_0$ are i.i.d. $W^l_{0,ij} \sim \mathcal{N}(0, \frac{\sigma_{w,l}^2}{n_l})$, and the biases as $b^l_{0,i} \sim \mathcal{N}(0, \sigma_{b,l}^2)$\footnote{Note, that $W^1x^0$ sums over $d$ terms, and for $l\ge 1$, $W^{l+1}x^l$ sums over $n$ terms.}. Under this parametrization, the norm of the Jacobian diverges for width $n_l \to \infty$. Without loss of generality, we will assume $n_1 = \ldots = n_L =: n$.

In the NTK parametrization, the layer outputs are scaled as:
\begin{equation}
h^{l+1} := \frac{1}{\sqrt{n_l}} W^{l+1}x^l +  b^{l+1}, \quad x^{l+1} := \phi(h^{l+1}) \in \mathbb{R}^{n_{l+1}},
\end{equation}
where the weights and biases are initialized i.i.d. as $W^l_{0,ij} \sim \mathcal{N}(0, \sigma_{w,l}^2)$ and $b^l_{0,i} \sim \mathcal{N}(0, \sigma_{b,l}^2)$. This scaling ensures stability in both the forward and backward pass for $n_l \to \infty$.

In Appendix \ref{appendix: standard and NTK parametrization} we will precisely show the equivalence between both parametrizations if one chooses the correct learning rate. Further, we will show that using the same learning rate for each parameter in standard parametrization leads to the first layer and the biases not being trained in the infinite-width limit. 

\paragraph{Neural Tangent Kernel.} The NTK characterizes the evolution of wide neural network predictions as a linear model in function space. Given a neural network $f(x, \theta)$ parameterized by $\theta \in \mathbb{R}^p$, define the parameter-Jacobian in any $N$ points $\mathbf{x}_1,\ldots, \mathbf{x}_N\in\mathbb{R}^d$ as $J(\mathbf{x},\theta) := \frac{\partial f(\mathbf{x}, \theta)}{\partial \theta} \in \mathbb{R}^{N \times p}$. Under NTK parametrization, the empirical NTK at two sets of inputs $\mathbf{x}'_1,\ldots,\mathbf{x}'_{N'}$ and $\mathbf{x}_1,\ldots,\mathbf{x}_N$ is defined as:
\begin{equation}
\label{eq-ntk}
\hat{\mathbf{\Theta}}_{\mathbf{x}', \mathbf{x}} := J(\mathbf{x}',\theta_0)J(\mathbf{x},\theta_0)^\top \in \mathbb{R}^{N'\times N}.
\end{equation}
The empirical NTK depends on the randomly initialized $\theta_0$. Interestingly, as shown by \citet{jacot2018neural}, in the infinite width limit, the empirical NTK converges to a deterministic (i.e., independent of $\theta_0$) kernel $\mathbf{\Theta}$ (the analytical NTK) at initialization and remains constant during training (with unregularized gradient flow). A GP defined by the analytical NTK is referred to as the NTK-GP.

\section{Motivation}
\label{sec-motivation}
Our work builds on a well-known result  \citep{lee2019wide} which shows that, for any sufficiently\footnote{In our proofs, we make the notion of `sufficiently wide' formal.} wide neural network initialized to parameters $\theta_0$, performing gradient descent on the loss
\vspace{-0.05in}
\begin{gather}
\label{eq-loss-vanilla}
\mathcal{L}^0(\theta) :=
\frac{1}{2} \sum_{i=1}^N (f(\mathbf{x}_i,\theta) - \mathbf{y}_i)^2
\end{gather}
until convergence gives us a network with predictions arbitrarily close to those given by the closed-form formula
\vspace{-0.1in}
\begin{gather}
\label{eq-gp-nonoise-initprior}
    f(\mathbf{x'},\theta_0) + \hat{\mathbf{\Theta}}_{\mathbf{{x',x}}}\hat{\mathbf{\Theta}}_{\mathbf{{x,x}}}^{-1}(\mathbf{y}-f(\mathbf{x},\theta_0)).
\end{gather}
This is of special interest because Equation \eqref{eq-gp-nonoise-initprior} is equivalent to computing the posterior mean of a Gaussian Process with the kernel $\hat{\mathbf{\Theta}}$, zero observation noise, and prior mean $f(\mathbf{x},\theta_0)$. In other words, there is a clear link between the computationally convenient process of gradient descent and a principled, Bayesian approach to inference. 

While the connection between Equations \eqref{eq-loss-vanilla} and \eqref{eq-gp-nonoise-initprior} is interesting, it is by itself not very practical. This is both because we seldom want to do inference for a randomly generated prior mean and because observations in regression tasks are often perturbed by noise. In this paper, we thus ask if it is possible to extend this equivalence by allowing for nonzero observation noise and for prior means other than $f(\mathbf{x},\theta_0)$. This leads us to the following research question.

\vspace{0.05in}
\begin{tcolorbox}[width=0.49\textwidth, before skip=0.5em, after skip=0.5em]
\paragraph{Main Research Question.} Is there a loss function and setting such that performing gradient descent on that loss gives us a network with predictions
\begin{gather*}
    m(\mathbf{x}') + \hat{\mathbf{\Theta}}_{\mathbf{{x',x}}}\left(\hat{\mathbf{\Theta}}_{\mathbf{{x,x}}}+ \beta I\right)^{-1}(\mathbf{y}-m(\mathbf{x})),
\end{gather*} 
for an arbitrary prior mean $m(\mathbf{x})$ and for arbitrary values of observation noise $\beta$?\end{tcolorbox}
\vspace{0.05in}

We answer this question in the affirmative. Specifically, Section \ref{sec-on} addresses limitation (i), observation noise, by showing that adding a regularizer to the training objective induces aleatoric noise $\beta$ in the NTK-GP posterior mean (Theorems \ref{thrm: first} and \ref{thrm: 3.4}), while Section \ref{sec-mu} addresses limitation (ii), arbitrary prior means, via a shifted network construction that enables user-specified priors (Theorem \ref{lemma-shifted-predictions}).

\section{Observation Noise through Regularized Gradient Descent}\label{sec-on}

The introduction of weight-space regularization not only impacts the generalization properties of wide neural networks \citep{hu2020simpleeffectiveregularizationmethods}, but also fundamentally alters their training dynamics. In the results proposed by \citet{lee2019wide} ---where the unregularized gradient flow is studied--- these dynamics are crucial to understanding how networks converge to solutions that (in the mean posterior sense) align with Bayesian inference. This section addresses limitation (i), observation noise, by analyzing the training dynamics of wide neural networks under regularized gradient flow, with a particular focus on how regularization affects network linearization during training. We demonstrate that the training trajectory stays arbitrarily close to its linearized counterpart, providing both theoretical justification and insights into how regularization modifies the gradient flow.

Define $f(\theta) := f(\mathbf{x},\theta)$, $g(\theta) := f(\mathbf{x},\theta) - \mathbf{y} \in \mathbb{R}^{N}$, $J(\theta) := J(\mathbf{x},\theta) \in \mathbb{R}^{N\times p}$ in the training points\footnote{We assume that $\mathbf{x}_i \neq \mathbf{x}_j$ for $i\neq j$.}.
We will consider (for NTK parametrization\footnote{See Appendix \ref{appendix: Standard param to NTK subappendix} for standard parametrization.}) the regularized training loss
\begin{align}\label{eq:loss}
    \mathcal{L}^{\beta}(\theta) 
    &:= \frac{1}{2} \sum_{i=1}^N (f(\mathbf{x}_i,\theta) - \mathbf{y}_i)^2 + \frac{1}{2}\beta \lVert \theta - \theta_0 \rVert_2^2 \\
    &= \frac{1}{2} \lVert g(\theta) \rVert_2^2 + \frac{1}{2} \beta \lVert \theta-\theta_0 \rVert_2^2. 
\end{align}
The gradient of this loss is given by
\begin{equation}
    \nabla_{\theta} \mathcal{L}^{\beta}(\theta) = J(\theta)^\top g(\theta) + \beta (\theta - \theta_0).
\end{equation}
We study the training dynamics under the regularized gradient flow\footnote{In Appendix \ref{appendix: gradient descent}, we state equivalent results for gradient descent with small enough learning rates. These will involve geometric sums instead of the exponential.}\footnote{We will prove that this ODE has a unique solution under our assumptions.} $\frac{d\theta_t}{dt} = -\eta \nabla_{\theta} \mathcal{L}^{\beta}(\theta_t).$
By the chain rule, $\frac{df(x,\theta_t)}{dt} = J(x,\theta_t)\frac{d\theta_t}{dt}$, and thus the training dynamics of the network are
{\small
\begin{equation}
    \frac{df(x,\theta_t)}{dt} = -\eta \left(J(x,\theta_t)J(\theta_t)^{\top}g(\theta_t) + \beta J(x,\theta_t)(\theta_t-\theta_0) \right).
\end{equation}
}

For the sake of readability, we omit the dependence of $\theta_t$ on $\theta_0$ and $\beta$. To gain insights into the role of regularization, we first analyze the regularized gradient flow for the linearized network
\begin{equation}
    f^{\lin}_{\theta_0}(x,\theta) := f(x,\theta_0) + J(x,\theta_0)(\theta-\theta_0).
\end{equation}
This is a linear ODE and hence has a closed-form solution. We formalize this in the theorem below.
\begin{theorem}\label{thrm: first}
 For training time $t\to\infty$, at any point $\mathbf{x'}$,

\vspace{-0.15in}
\begin{small}
\begin{equation}\label{eq: thrm-4.1}
    f^{\lin}_{\theta_0}(\mathbf{x'},\theta_\infty) = f(\mathbf{x'},\theta_0) + \hat{\mathbf{\Theta}}_{\mathbf{{x',x}}}\left(\hat{\mathbf{\Theta}}_{\mathbf{{x,x}}}+ \beta I\right)^{-1}(\mathbf{y}-f(\mathbf{x},\theta_0)).
\end{equation}
\end{small}
\vspace{-0.2in}
\end{theorem}
We derive this closed-form solution (for gradient flow and gradient descent) in Appendix \ref{appendix: regularized gf/gd for linearized network}. Note that the statements are high-probability results with respect to the random initialization: for any $\delta_0 > 0$, the probability can be made at least $1-\delta_0$ by choosing the width $n$ sufficiently large. The required scaling of $n$ with $\delta_0$ is made precise in Appendices \ref{app: jacobian-app} and \ref{appendix: Proof regularized gradient flow}.

Moreover, if we consider the network initialization as a random variable, the trained linearized network converges to a normal distribution for layer width $n\to\infty$. Its mean corresponds to the posterior mean of the NTK-GP with prior mean $0$. However, similarly to \citet{lee2019wide, he2020bayesian}, the covariance cannot be interpreted as the posterior covariance of an NTK-GP. See Appendix \ref{appendix: gaussian-distribution-convergence} for the exact formulas.

In the remainder of this section, we show that, with high probability over initialization, the training dynamics of the network (at any test point $\mathbf{x'}$) closely follow the output of the linearized network for large enough hidden layer width $n$.
First, we leverage the fact that the Jacobian is locally bounded and locally Lipschitz, with a Lipschitz constant of $O\left(\frac{(\log n)^c}{\sqrt{n}}\right)$ (see Appendix \ref{app: jacobian-app} for details). Next, we prove that the distance between the parameters and their initialization $\lVert \theta_t - \theta_0 \rVert_2$ is $O(1)$\footnote{On average, this implies that the distance of individual parameters barely changes.}. This step is particularly challenging because proofs for $\beta = 0$ (in Eq. \eqref{eq:loss}) rely on the MSE $\lVert g(\theta_t)\rVert_2$ decaying exponentially to $0$. As this is not the case for $\beta > 0$, we instead show that $\lVert \nabla_{\theta} \mathcal{L}^{\beta}(\theta_t) \rVert_2$ decays to $0$ exponentially fast.

Even for $\beta = 0$, our analysis provides a generalization of previous proofs in the literature and presents further insights into why the network converges to a global minimizer, despite the Hessian of the loss having negative eigenvalues. Finally, we conclude that the regularized gradient flow remains arbitrarily close to its linearized counterpart, and hence the same holds for the network.

\subsection{Exponential Decay of the Regularized Gradient and Parameter Proximity to Initialization}
Building upon the established Lipschitz continuity of the Jacobian (in Appendix \ref{app: jacobian-app}), we analyze the effect of regularized gradient flow on the norm of the gradient and the parameters' closeness to their initialization. We begin by demonstrating that the norm of the gradient of the regularized loss decays exponentially over time. This result directly implies that the distance between the parameters and their initialization is bounded by a constant, i.e., $\lVert \theta_t - \theta_0 \rVert_2 = O(1)$\footnote{The bounding constant is independent of the outcome in the high probability event for which this holds.}. Furthermore, leveraging the Lipschitzness of the Jacobian established in Lemma \ref{lemma: Jacobian Lipschitz}, we show that the Jacobian remains close to its initial value, with deviations scaling as $O\left(\frac{(\log n)^c}{\sqrt{n}}\right)$.

\begin{restatable}{theorem}{gftheorempartone}
\label{theorem: Exponential decay gradient, parameters stay close}
Let $\beta \ge 0$. Let $\delta_0>0$ arbitrarily small. There are $K',K, R_0, c_{\beta} > 0$, such that for $n$ large enough, the following holds with probability of at least $1-\delta_0$ over random initialization, when applying regularized gradient flow with learning rate $\eta = \eta_0$:
\begin{equation}
    \norm{ \frac{d\theta_t}{dt}}_2 
    = \eta_0 \lVert \nabla_{\theta} \mathcal{L}^{\beta}(\theta_t) \rVert_2
    \le \eta_0 K R_0 e^{-\eta_0c_{\beta} t},
\end{equation}
\begin{equation}
    \lVert \theta_t - \theta_0 \rVert_2 
    \le \frac{KR_0}{c_{\beta}} \left(1 - e^{-\eta_0 c_{\beta} t} \right) < \frac{KR_0}{c_{\beta}} =: C,
\end{equation}
\begin{equation}
    \forall \lVert x \rVert_2 \le 1: \lVert J(x,\theta_t) - J(x,\theta_0) \rVert_2 \le \frac{(\log n)^c}{\sqrt{n}} K' C,
\end{equation}
\begin{equation}
    \lVert J(\theta_t) - J(\theta_0) \rVert_2 \le\frac{(\log n)^c}{\sqrt{n}}KC.
\end{equation}
\end{restatable}
We provide a proof in Appendix \ref{appendix: proof regularized gradient flow part 1 (exponential decay)}.
The main idea is to analyze $\frac{d}{dt}\lVert \nabla_{\theta} \mathcal{L}^{\beta}(\theta_t)\rVert_2$ and leverage the fact that the eigenvalues of the Hessian of the network scale as $O\left(\frac{(\log n)^c}{\sqrt{n}}\right)$. Consequently, these eigenvalues are dominated by the regularization term $\beta > 0$, leading to the exponential decay. This is a substantial shift from prior proofs for the case $\beta = 0$, where exponential decay is shown for $\lVert g(\theta_t) \rVert_2$ (the training error) instead\footnote{This implies exponential decay of $\lVert \nabla_{\theta} \mathcal{L}^{0}(\theta_t) \rVert_2 =\lVert J(\theta_t)^{\top} g(\theta_t) \rVert_2 $.}.

Our result for $\beta > 0$ does not require the minimum eigenvalue $\lambda_{\min}(\mathbf{\Theta})$ of the analytical NTK to be positive. For the unregularized case $\beta=0$ (and $\lambda_{\min}(\mathbf{\Theta})>0$), we recover the results of \citet{lee2019wide} by an alternative proof. This highlights the importance of the gradient being in the row-span of the Jacobian $J(\theta_t)$, ensuring that only the positive eigenvalues of $J(\theta_t)^\top J(\theta_t)$ contribute to the dynamics.

\begin{restatable}{theorem}{gftheoremparttwo}
\label{thrm: 3.4}
Let $\beta \ge 0$. Let $\delta_0 > 0$ be arbitrarily small. Then, there are $C_1, C_2$, such that for $n$ large enough, with probability of at least $1-\delta_0$ over random initialization,
\begin{equation}
    \sup_{t\ge 0} \lVert \theta_t - \theta^{\lin}_t \rVert_2 \le C_1 \frac{(\log n)^c}{\sqrt{n}},
\end{equation}
\begin{equation}
    \forall \lVert x \rVert_2 \le 1: \sup_{t\ge 0} \lVert f(x,\theta_t) - f^{\lin}_{\theta_0}(x,\theta_t^{\lin}) \rVert_2 \le C_2 \frac{(\log n)^c}{\sqrt{n}} .
\end{equation}
\end{restatable}
Where $\theta^{\lin}_t$ are the parameters at time $t$ of the linearized network $f^{\lin}_{\theta_0}$. See Appendix \ref{appendix: Proof regularized gradient flow part 2 (closeness linearized)} for a proof. Again, unlike previous proofs for $\beta = 0$, which rely on the exponential decay of $\lVert g(\theta_t)\rVert_2$, we present a more general and simpler approach that applies for $\beta \geq 0$. Specifically, we decompose $\lVert f(x,\theta_t) - f^{\lin}_{\theta_0}(x,\theta_t^{\lin}) \rVert_2$ into $\lVert f(x,\theta_t) - f^{\lin}_{\theta_0}(x,\theta_t) \rVert_2$ and $\lVert f^{\lin}_{\theta_0}(x,\theta_t) - f^{\lin}_{\theta_0}(x,\theta_t^{\lin}) \rVert_2$, and bound them using Theorem \ref{theorem: Exponential decay gradient, parameters stay close}.

The results in this section establish the key conditions required to conclude that wide neural networks under regularized gradient flow remain in a linear regime throughout training. This enables us to use Theorem \ref{thrm: first} to analyze network training with the regularized loss.

\subsection{NTK-GP Posterior Mean with Aleatoric Noise Interpretation}
Given the linearization of our network, we can now look back to Theorem \ref{thrm: first} and observe that, at convergence under regularized gradient flow and given a random initialization, the output of the linearized network corresponds to the posterior mean of a GP with the analytical NTK as the kernel and non-zero observation noise. This provides a direct Bayesian interpretation of the regularized training process in terms of the NTK-GP framework. 

It is important to highlight that the trained parameters $\theta_\infty$ generally depend on the specific initialization $\theta_0$, and as such, will differ for every random initialization. This variability is relevant to the training process but is not central to our analysis and use case. Unlike deep ensembles \citep{he2020bayesian} (particularly when used in the context of Thompson sampling \citep{thompson-sampling}), where the distribution of trained networks across initialization is explicitly used, our focus lies on the behaviour of a single trained network under a given initialization. We are able to alleviate the effect of this randomness thanks to our initialization strategy described in the next section.

\section{Neural Network Initialization as NTK-GPs with Arbitrary Prior Mean}\label{sec-mu}

This section addresses limitation (ii), arbitrary prior means, by introducing a shifted network construction that enables deterministic NTK-GP posterior means with user-specified priors. While standard initialization schemes for neural networks typically involve randomly setting weights and biases to break symmetry, they generally lack a principled way to align the network's prior outputs with specific inductive biases or desired properties. In contrast, Gaussian processes offer a well-defined framework for specifying prior distributions over functions through the prior mean and covariance. Arbitrary initialization is especially valuable in reinforcement learning, where optimistic priors aid exploration \citep{strehl2006pac}, and in domains with strong prior knowledge, where encoding a meaningful prior mean can improve performance.

One way to use the results presented in the previous section to obtain the posterior of the NTK-GP with zero prior mean is to train several ensembles and average the outputs using the law of large numbers, akin to the approach by \citet{he2020bayesian}. However, if we are only interested in computing the posterior mean and not the covariance, there is a better way, which only requires the use of one network. Section \ref{sec-shift} explores how neural networks, under the NTK-GP framework, can be initialized to reflect arbitrary prior means, enhancing their flexibility to align with specific tasks. Appendix \ref{app:no-prior-mean} shows why simply setting the last layer’s weights to zero does not yield the desired prior.

\subsection{Shifting the labels or predictions}
\label{sec-shift}
 
Inspired by standard techniques in GP literature \citep{williams2006gaussian}, we provide a formal construction for modifying the network to introduce arbitrary prior mean. This can be accomplished either by shifting the predictions of the neural network by its predictions at initialization, or equivalently, defining a new shifted network. The intuition behind this is that the shifting does not change the Jacobian of the network with respect to the parameters, but gives a different initialization in Equation \eqref{eq: thrm-4.1}. The following theorem formalizes this, resolving the main research question we posed in Section \ref{sec-motivation}.

\begin{restatable}{theorem}{shiftpredictions}
\label{lemma-shifted-predictions}
(Shifted Network.) Consider any function $m$. Given a random initialization $\theta_0$, define shifted predictions $\Tilde{f}_{\theta_0}(\mathbf{x}, \theta)$ as follows:
\begin{equation}
    \Tilde{f}_{\theta_0}(\mathbf{x}, \theta) := f(\mathbf{x}, \theta) - f(\mathbf{x}, \theta_0) + m(\mathbf{x}).
\end{equation}
Training this modified network (starting with $\theta_{0}$) leads to the following output (in the infinite-width limit)
\begin{equation}
    \tilde{f}_{\theta_0}(\mathbf{x'}, \theta_\infty) = m(\mathbf{x'}) + \mathbf{\Theta}_{\mathbf{x', x}}(\mathbf{\Theta}_{\mathbf{x, x}} + \beta I)^{-1}(\mathbf{y} - m(\mathbf{x})).
\end{equation}
This can be interpreted as the posterior mean of an NTK-GP with prior mean function $m$.
\end{restatable}
We prove this in Appendix \ref{appendix: Shifting the network at initialization}.
Note that $\tilde{f}_{\theta_0}(\mathbf{x'},\theta_{\infty})$ is non-random (i.e., independent of  the initialization $\theta_0$). This is different from training the non-shifted network, where $f_{\theta_0}(\mathbf{x'},\theta_{\infty})$ is random, but has a mean which corresponds to the posterior mean of an NTK-GP with prior mean $0$. Note, that the non-randomness of $\tilde{f}_{\theta_0}(\mathbf{x'},\theta_{\infty})$ is not immediately clear: while $\tilde{f}_{\theta_0}(x,\theta_0)$ does not depend on $\theta_0$, $\tilde{f}_{\theta_0}(x,\theta)$ differs for different $\theta_0$ when $\theta\neq \theta_0$. The reason for the non-randomness after training is that the NTK doesn't depend on the initialization in the infinite-width limit.

One drawback of shifting predictions is the need to store the initial parameters $\theta_0$, effectively doubling the memory requirements, which can be prohibitive for large networks. Additionally, performing a forward pass of the initial network to compute the shift adds a slight additional computational overhead. However, it is worth noting that no additional back-propagation is required through $f(\mathbf{x}, \theta_0)$ or $m(\mathbf{x})$, keeping the overall cost manageable and lower than even a two-network ensemble.

\section{Experiments}
In this section, we empirically validate our theoretical results by studying the convergence of wide neural networks (and parameters) trained with regularized gradient descent to their linearized counterparts. First, we analyze how the trained parameters deviate from the optimal linearized network parameters (i.e., the converged linear regression weights) and how this difference shrinks as the width increases. We then compare the network's predictions with those of kernel ridge regression on a test set, varying network depth and the regularization strength $\beta$. Lastly, we show the effectiveness of using a pre-trained neural network as an informative prior mean in terms of performance across different dataset sizes in a synthetic task.

\subsection{Experimental Setup}

\begin{figure*}[h!]
    \centering
    \includegraphics[width=0.32\textwidth]{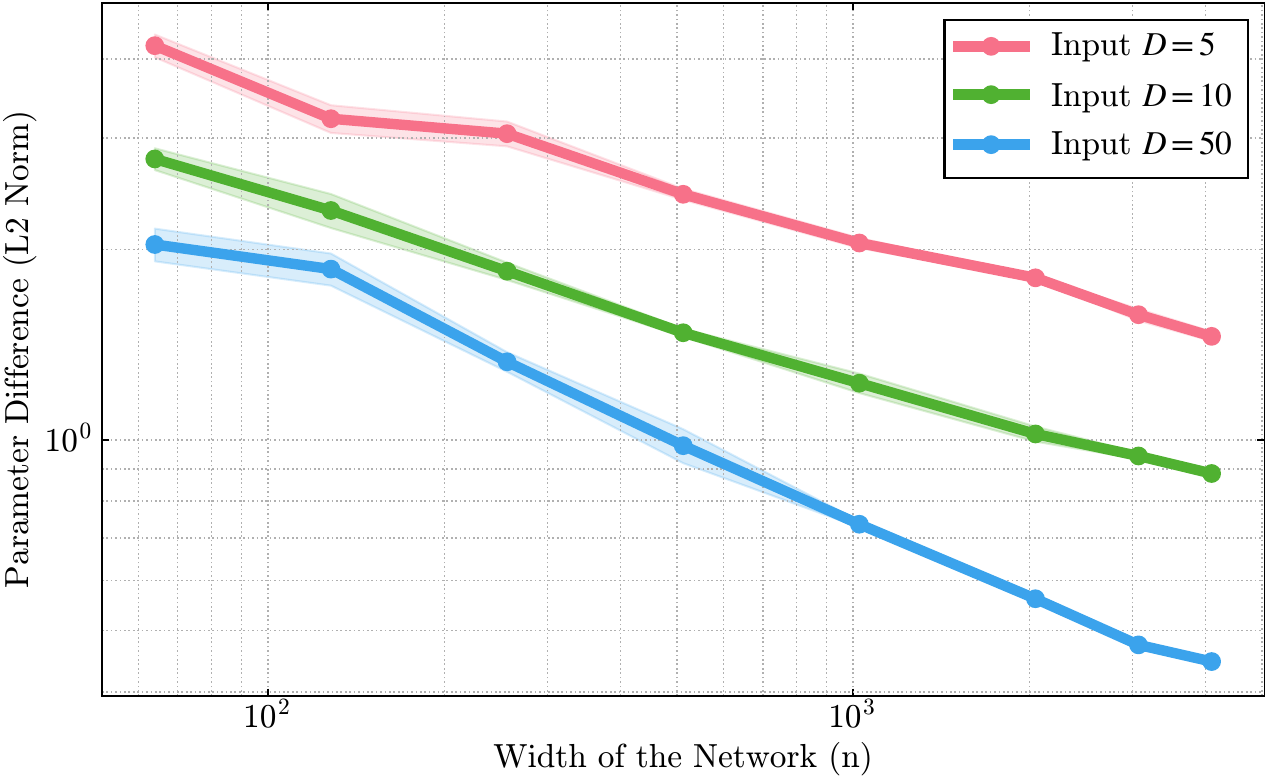}
    \includegraphics[width=0.32\textwidth]{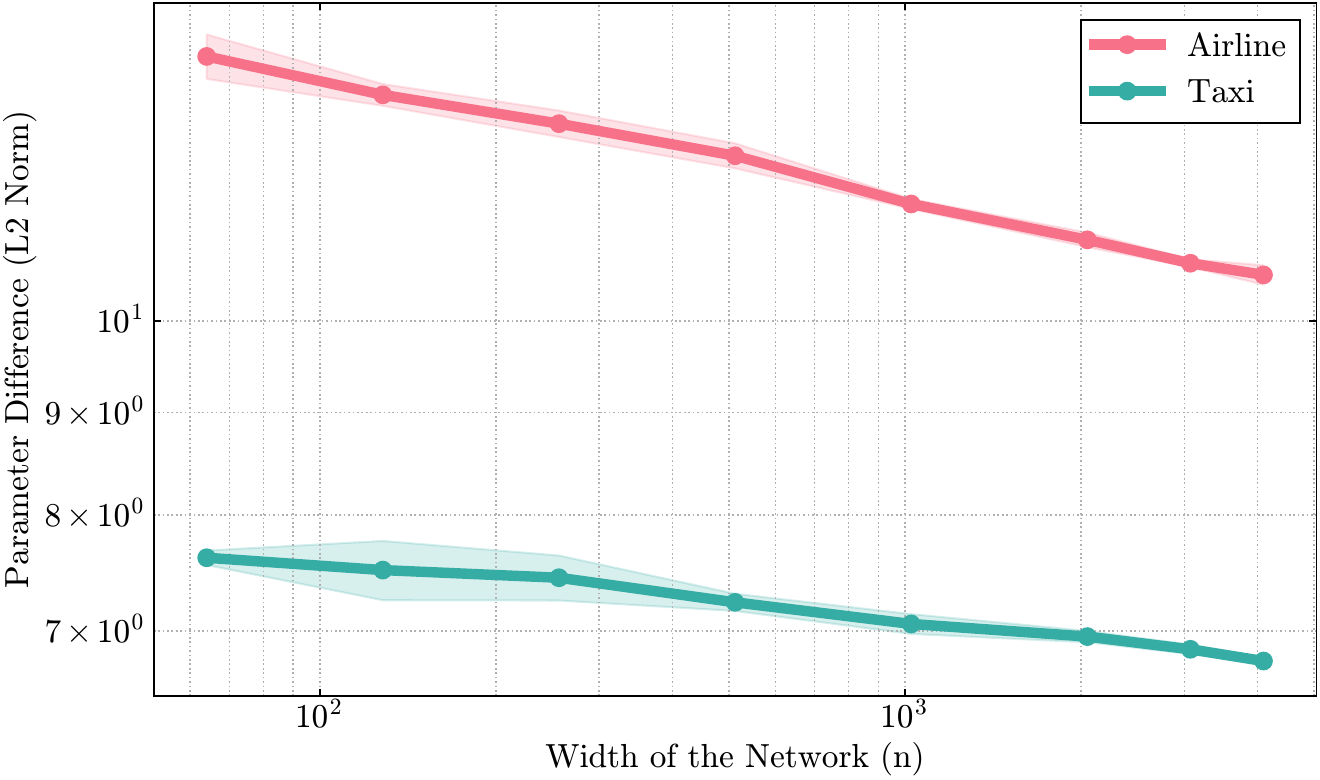}
    \includegraphics[width=0.32\textwidth]{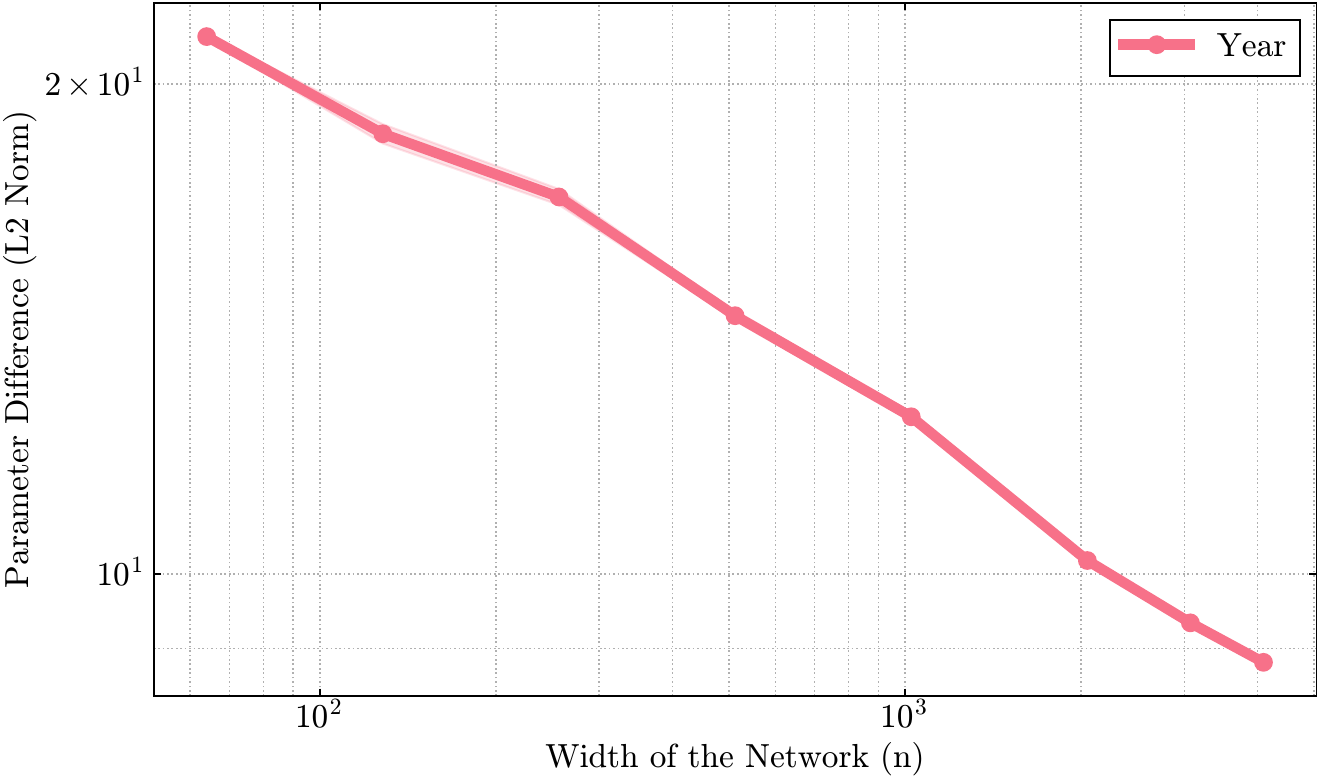}
    \caption{Frobenius norm differences between the trained neural network's parameters and the kernel ridge regression solution plotted against network width for different datasets. \textbf{Left:} Different input dimensions of the synthetic dataset. \textbf{Middle:} \textit{Airline} and \textit{Taxi} datasets. \textbf{Right:} UCI \textit{Year} dataset. Shaded regions represent the standard deviation divided by the square root of the number of seeds.}
    \label{fig:param_differences}
\end{figure*}

\begin{figure*}[h!]
    \centering
    \includegraphics[width=0.32\textwidth]{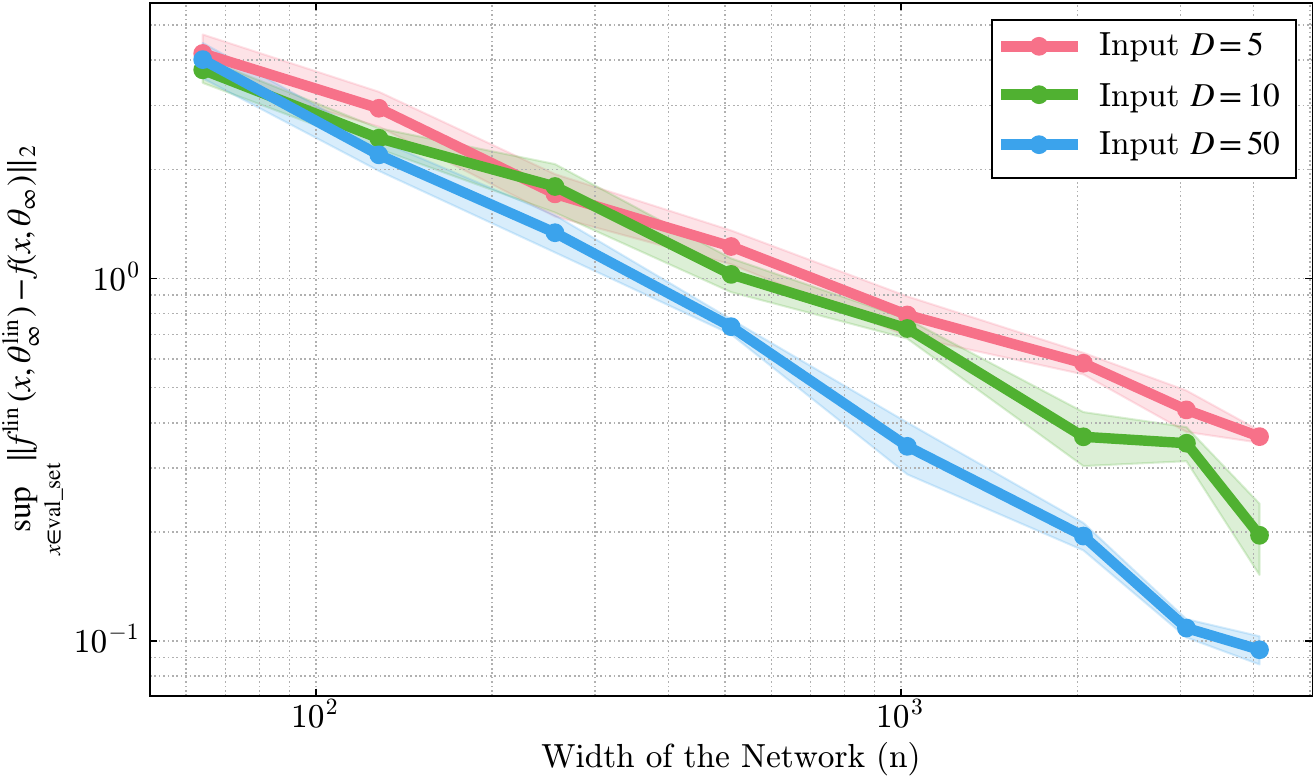}
    \includegraphics[width=0.32\textwidth]{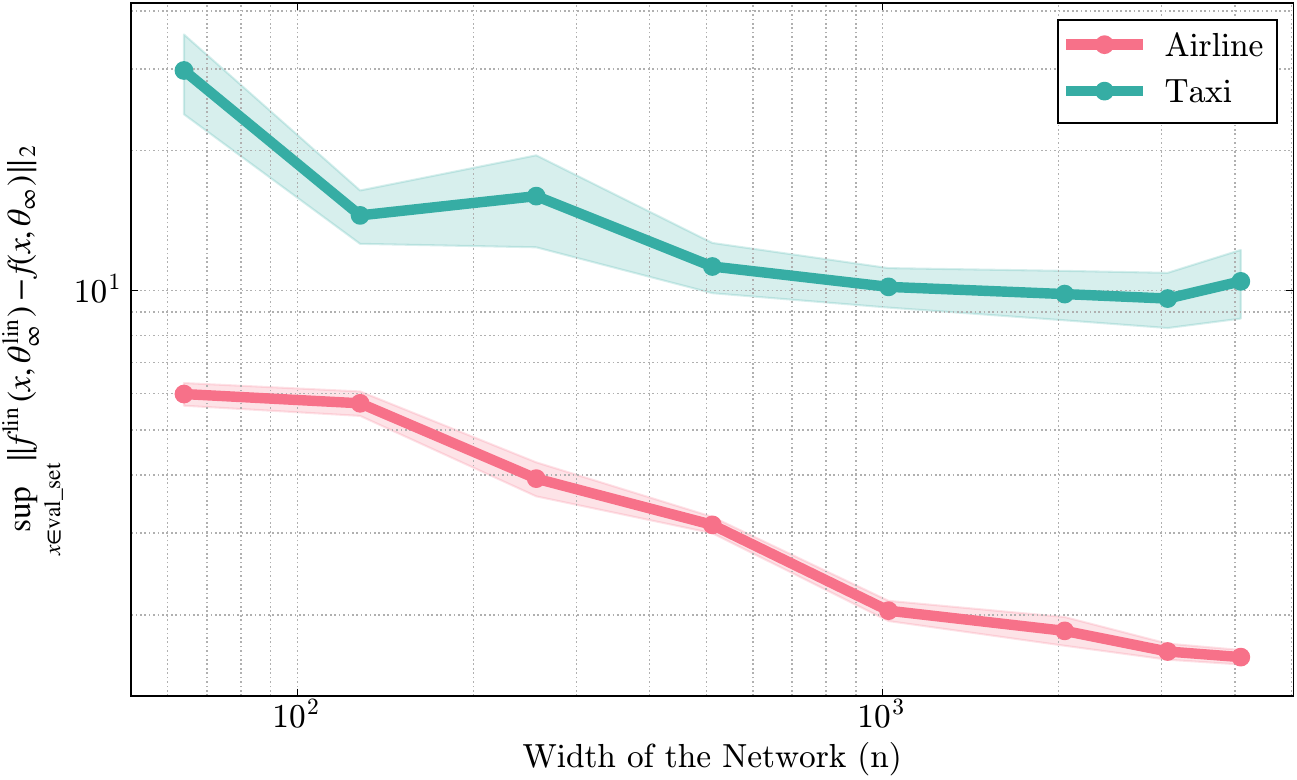}
    \includegraphics[width=0.32\textwidth]{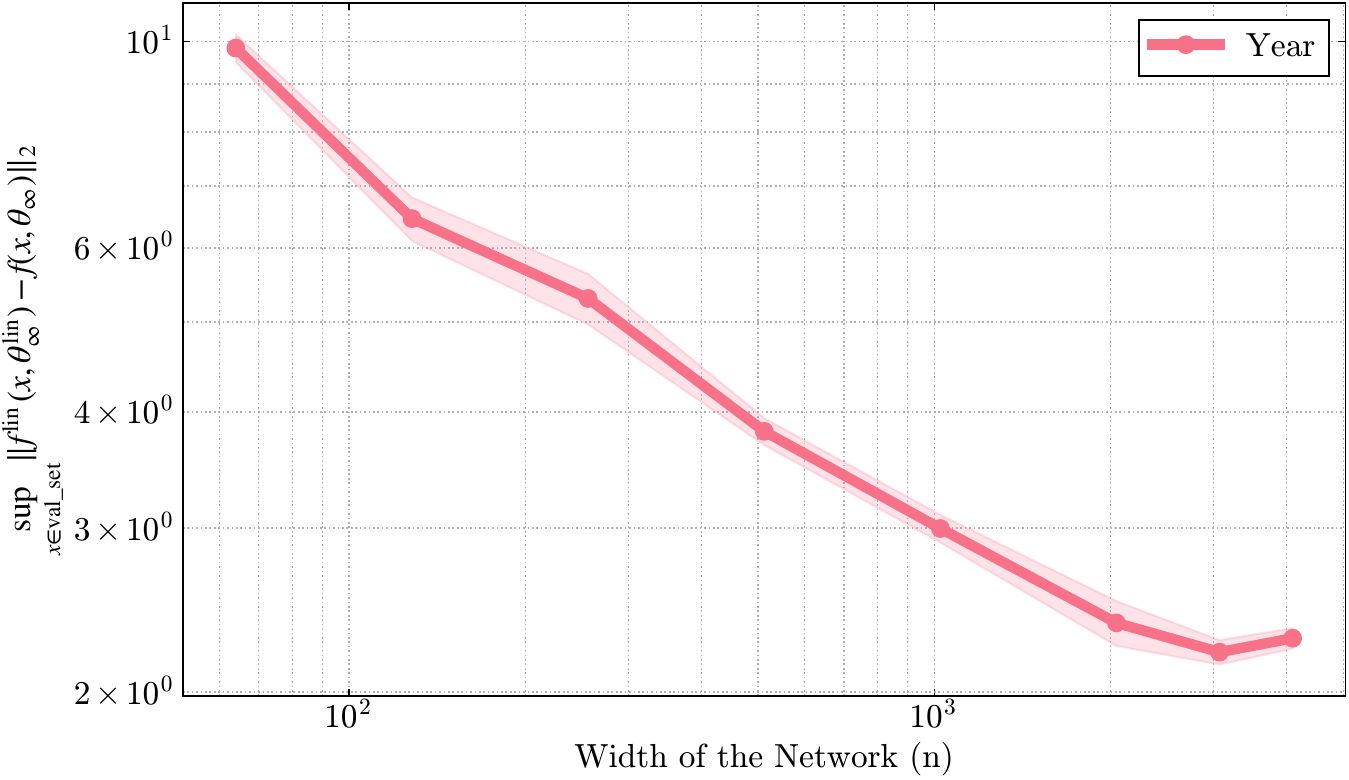}
    \caption{Supremum of the $\ell_2$ norm between the trained neural network's outputs \(f(x, \theta_\infty)\) and the linearized model's predictions \(f^{\text{lin}}(x, \theta_\infty^{\text{lin}})\) across the validation set, plotted against network width. \textbf{Left:} Different input dimensions of the synthetic dataset. \textbf{Middle:} \textit{Airline} and \textit{Taxi} datasets. \textbf{Right:} UCI \textit{Year} dataset.\vspace{-0.1in}}
    \label{fig:functions-diff}
\end{figure*}

To ensure alignment with the NTK framework, we train wide fully connected multi-layer perceptrons (MLPs) under the NTK parametrization using full-batch gradient descent (additional results under SGD with different batch sizes are included in Appendix \ref{app: additional-experiments}). The experiments are performed in both synthetic and standard regression datasets. The synthetic problem is a multi-dimensional synthetic regression task where the target function is defined as $\mathbf{y} = \sin(\mathbf{x}) + \cos(2\mathbf{x}) + \boldsymbol{\epsilon}$, with $\boldsymbol{\epsilon} \sim \mathcal{N}(0, I\sigma^2)$ representing the observation noise. Inputs $\mathbf{x}$ are uniformly sampled from the range $[-6, 6]$. We generate a dataset of a few hundred training points and use a two-layer neural network. 
For the standard regression tasks, we use the datasets \textit{Airline} \citep{hensman2013gaussian}, \textit{Taxi} \citep{peng2017asynchronous, salimbeni2017doubly}, and UCI \textit{Year} (\citet{hernandez2015probabilistic}), which are complex and high-dimensional. We follow \citet{bergna2025post} for architectural choices and hyperparameter settings. All experiments were performed on a single NVIDIA L40 GPU.

To compare the trained network parameters to their linearized counterparts, we compute the Frobenius norm between the network parameters and those obtained via the kernel ridge regression using the NTK (Section \ref{sec: 5.2}). Similarly, we evaluate the deviation between the trained neural network and the kernel ridge regression predictions on a test set by computing the squared error (Section \ref{sec: 5.3}). We repeat these experiments over 5 different seeds and plot the mean and standard error on a logarithmic scale. The plots below in the synthetic experiments are for 3-layer MLPs with $\beta=0.5$. Results for different depths and $\beta$ coefficients are shown in linear plots in Appendix \ref{app: additional-experiments}.

Finally, in Section \ref{sec: prior-mean}, we explore how incorporating a learned prior mean can improve data efficiency in transfer settings. We compare training from scratch to initializing with the posterior from a related task, and observe significantly improved performance, especially in low-data regimes.

\subsection{Empirical Convergence of the Parameters}\label{sec: 5.2}

We examine how the trained network parameters compare to those obtained via kernel ridge regression with the NTK. Specifically, we compute the $\ell_2$ norm difference between the final parameters of the trained network and the corresponding optimal linearized solution. Leveraging Theorem \ref{lemma-shifted-predictions}, we set the prior mean to zero for all experiments. Results in Figure \ref{fig:param_differences} confirm that this difference decreases as the network width increases, supporting the theoretical prediction offered in Theorem \ref{thrm: 3.4}. 


At smaller widths, we observe a notable deviation between the trained network and the linear regression solution, realizing the effect of finite width. As the width increases, the difference gradually declines, indicating the convergence towards the linearized model post-training in all of the tested datasets. We note that computing and inverting the NTK matrix to compute the reference linear model presented a significant computational bottleneck, making it infeasible to scale to larger widths, depths, or dataset sizes. In contrast, training the neural network was substantially more efficient.

\subsection{Empirical Convergence of the Trained Network to a Linear Model}\label{sec: 5.3}

We now shift from parameter convergence to function-space convergence. Here we evaluate how closely the trained neural network approximates the predictions of the corresponding linearized model for unseen data. This validation dataset was sampled from the same data-generating process as the training dataset, comprising 20\% of its size. Specifically, we compute $\sup_{x \in \mathcal{V}} \lVert f(x, \theta_\infty) - f^{\lin}(x, \theta_\infty^{\lin}) \rVert_2,$
where $\mathcal{V}$ is the validation set. This represents the distance between the trained network's outputs and the kernel ridge regression solution evaluated in a validation set.


Every plot in Figure \ref{fig:functions-diff} shows, under a similar setup as in Section \ref{sec: 5.2}, that as the width increases, this discrepancy diminishes, providing further empirical confirmation of Theorem \ref{thrm: 3.4}. As in Section \ref{sec: 5.2}, more plots for different depths can be found in Appendix \ref{app: additional-experiments}.

\subsection{Pre-Trained Neural Network as Prior Mean}\label{sec: prior-mean}

In this experiment, we construct two $N$-dimensional regression tasks that share low-frequency structure but differ in their high-frequency components. In both cases, observations follow $\mathbf{y} = f(\mathbf{x}) + \boldsymbol{\epsilon}$, where $\boldsymbol{\epsilon} \sim \mathcal{N}(0, \sigma^2I)$. Task 1 uses the target function $f_1(\mathbf{x}) = \sin(\mathbf{x}) + 0.5\sin(5\mathbf{x}) + 0.2\sin(20\mathbf{x})$, while Task 2 is defined by $f_2(\mathbf{x}) = \sin(\mathbf{x}) + 0.3\cos(7\mathbf{x}) - 0.2\sin(15\mathbf{x})$. Coefficients were chosen arbitrarily and were not tuned.

\begin{figure}[h!]
    \centering
    \includegraphics[width=0.8\linewidth]{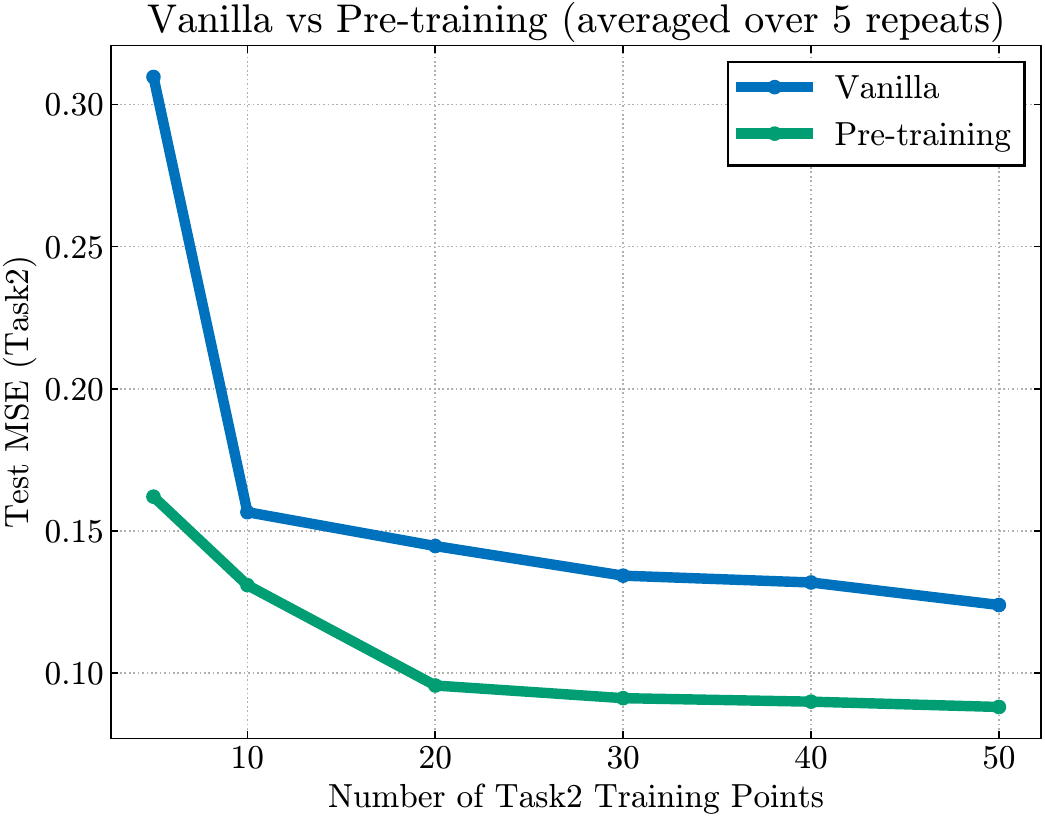}
    \caption{Test MSE on Task 2 as a function of the number of training points, comparing training from scratch (Vanilla) with using the posterior from Task 1 as a prior mean (Pre-training). Pre-training significantly improves performance in low-data regimes.\vspace{-0.1in}}
    \label{fig:prior-mean}
\end{figure}

We compare two learning strategies: (i) \textbf{Vanilla:} initialize a two-layer Softplus network at random and train on Task~2 alone, with zero prior mean. (ii) \textbf{Pre-training:} first train the same network on Task~1 to convergence, then use its learned posterior mean as the prior when fitting Task~2. Figure~\ref{fig:prior-mean} shows the test MSE on Task~2 as a function of $n_2$. The pre-training curve lies strictly below the vanilla curve---most notably in the low-data regime ($n_2 \le 20$)---demonstrating that the Task~1 posterior provides a powerful inductive bias that reduces sample complexity on Task~2.

\section{Related Work}
\paragraph{Linearization} \citet{lee2019wide} demonstrate that as network width approaches infinity, training dynamics simplify and can be approximated by a linearized model using a first-order Taylor expansion. \citet{lee2019wide} also study the links between the output of trained neural networks and GPs. Crucially, we extend this work by proving that this linearization still holds in the presence of observation noise.

\paragraph{Kernel Methods and Neural Networks}
\citet{jacot2018neural} established the equivalence between wide neural networks and kernel methods via kernel ridge regression, and popularized studying neural networks in function space rather than parameter space. We build on both ideas: Theorem \ref{lemma-shifted-predictions} adopts the function-space view of GPs. Subsequent work has further explored NTK–GP connections in noisy settings. For instance, \citet{rudner2023functionspaceregularizationneuralnetworks} proposed function-space regularization to encode desired properties into predictions, while \citet{chen2022neural} related NTK norms to RKHS regularization.

\paragraph{Global Minima and Overparameterization} In overparameterized settings, \citet{allenzhu2019convergence} show that SGD can reach global minima in polynomial time, while \citet{zou2020gradient} prove convergence for ReLU networks by keeping weights close to their initialization. Although we do not directly rely on these results, we also guarantee high-probability convergence to the global optimum.

\section{Conclusion and Future Work}

We studied the relationship between regularized training in wide neural networks and Gaussian Processes under the NTK framework. We are the first work formally showing that the proposed weight-space regularization in neural networks is equivalent to adding aleatoric noise to the NTK-GP posterior mean in the infinite-width limit. Additionally, we introduced a shifted network approach that enables arbitrary prior functions and ensures deterministic convergence to the NTK-GP posterior mean without requiring ensembles or kernel inversion. Empirical results validate our theoretical findings, demonstrating convergence to the linearized network. The methods proposed in this paper establish a foundation for practical Gaussian Process inference using the NTK and standard neural network training. Future work could explore the benefits of leveraging the NTK to efficiently train GPs with interpretable and widely used kernels, such as RBF or Matérn. This would require looking into algorithms to compute the posterior covariance \citep{calvo2024epistemic, zanger2025contextualsimilaritydistillationensemble}.

\section*{Acknowledgements}

SCO's research is supported by the Oxford-Man Institute through the EPSRC Centre for Doctoral Training in Mathematics of Random Systems: Analysis, Modelling and Simulation (EPSRC Grant EP/S023925/1). JP acknowledges financial support from the Oxford-Man Institute. SCO, JP, and AC acknowledge the support of the Oxford-Man Institute for providing computational resources. JMHL acknowledges support from EPSRC funding under grant EP/Y028805/1. JMHL also acknowledges support from a Turing AI Fellowship under grant EP/V023756/1.

{\small
\bibliographystyle{unsrtnat}
\bibliography{neurips_2025}
}

\section*{Checklist}



\begin{enumerate}

  \item For all models and algorithms presented, check if you include:
  \begin{enumerate}
    \item A clear description of the mathematical setting, assumptions, algorithm, and/or model. Yes.
    \item An analysis of the properties and complexity (time, space, sample size) of any algorithm. Yes.
    \item (Optional) Anonymized source code, with specification of all dependencies, including external libraries. We will provide this upon acceptance.
  \end{enumerate}

  \item For any theoretical claim, check if you include:
  \begin{enumerate}
    \item Statements of the full set of assumptions of all theoretical results. Yes.
    \item Complete proofs of all theoretical results. Yes.
    \item Clear explanations of any assumptions. Yes.
  \end{enumerate}

  \item For all figures and tables that present empirical results, check if you include:
  \begin{enumerate}
    \item The code, data, and instructions needed to reproduce the main experimental results (either in the supplemental material or as a URL). We will provide the code upon acceptance. The data and instructions are provided.
    \item All the training details (e.g., data splits, hyperparameters, how they were chosen). Yes.
    \item A clear definition of the specific measure or statistics and error bars (e.g., with respect to the random seed after running experiments multiple times). Yes.
    \item A description of the computing infrastructure used. (e.g., type of GPUs, internal cluster, or cloud provider). Yes. 
  \end{enumerate}

  \item If you are using existing assets (e.g., code, data, models) or curating/releasing new assets, check if you include:
  \begin{enumerate}
    \item Citations of the creator If your work uses existing assets. Yes.
    \item The license information of the assets, if applicable. Not Applicable.
    \item New assets either in the supplemental material or as a URL, if applicable. Not Applicable.
    \item Information about consent from data providers/curators. Not Applicable.
    \item Discussion of sensible content if applicable, e.g., personally identifiable information or offensive content. Not Applicable.
  \end{enumerate}

  \item If you used crowdsourcing or conducted research with human subjects, check if you include:
  \begin{enumerate}
    \item The full text of instructions given to participants and screenshots. Not Applicable.
    \item Descriptions of potential participant risks, with links to Institutional Review Board (IRB) approvals if applicable. Not Applicable.
    \item The estimated hourly wage paid to participants and the total amount spent on participant compensation. Not Applicable.
  \end{enumerate}

\end{enumerate}

\clearpage
\appendix
\onecolumn

\section{Local Boundedness and Lipschitzness of the Jacobian}\label{app: jacobian-app}
To analyze the training dynamics of wide neural networks, it is essential to ensure that the Jacobian of the network remains (locally) bounded and Lipschitz-continuous during training, with the constants scaling ``well" with layer width $n$. This is shown by the following Lemma, which was proved in \citet{lee2019wide} for standard parametrization, and more formally in \citet{liu2020linearity} for NTK parameterization.

\begin{restatable}{lemma}{lipschitznesslemma}
\label{lemma: Jacobian Lipschitz}
For any $\delta_0>0$, there is $K'>0$ (independent of $C$), such that: For every $C>0$, there is $n$ large enough, such that with probability of at least $1-\delta_0$ (over random initialization): For any point $x$ with $\lVert x \rVert_2 \le 1$:
\begin{equation}
    \forall \theta \in B(\theta_0,C): \lVert J(x,\theta)\rVert_2 
    \le K',
\end{equation}
\begin{equation}
    \forall \theta,\tilde{\theta} \in B(\theta_0, C): \lVert J(x,\theta) - J(x,\tilde{\theta}) \rVert_2 
    \le \frac{(\log n)^c}{\sqrt{n}}K' \lVert \theta - \tilde{\theta} \rVert_2.
\end{equation}
In particular, with $K=\sqrt{N}K'$, for the Jacobian over the training points:
\begin{equation}
    \forall \theta \in B(\theta_0,C): \lVert J(\theta)\rVert_F \le K,
\end{equation}
\begin{equation}
    \forall \theta,\tilde{\theta} \in B(\theta_0, C): \lVert J(\theta) - J(\tilde{\theta}) \rVert_F \le \frac{(\log n)^c}{\sqrt{n}}K \lVert \theta - \tilde{\theta} \rVert_2.
\end{equation}
As a direct consequence, for the Hessian $\nabla^2_{\theta} f(x, \theta) \in \mathbb{R}^{p \times p}$ of the network,
\begin{equation}
    \forall \theta \in B(\theta_0,C): \norm{ \nabla^2_{\theta} f(x, \theta) }_2 \le \frac{(\log n)^c}{\sqrt{n}} K'.
\end{equation}
\end{restatable}
The Frobenius norm is used to aggregate over different training points, i.e. $\lVert J(\theta)\rVert_F^2 = \sum_{i=1}^N \lVert J(\mathbf{x}_i, \theta) \rVert_2^2$ for $J(\theta) = J(\mathbf{x},\theta)\in\mathbb{R}^{N\times p}$.
The logarithmic term $(\log n)^c$ is sometimes omitted in the literature. Its power $c>0$ is a constant that only depends on the network architecture.

\section{Initializing the last layer weights as $0$ does not lead to a zero prior mean NTK-GP}\label{app:no-prior-mean}
In this section, we explore a tempting but ultimately futile way of constructing a network with zero prior mean. One might be tempted to just initialize the last layer of the network to $0$, i.e., $\sigma_{w, L+1} = 0$, $\sigma_{b, L+1} = 0$. Then, the gradient with respect to any weights/biases not in the last layer will be $0$. In this case, the empirical NTK is given by
\begin{equation}
    \hat{\mathbf{\Theta}}_{\mathbf{x}',\mathbf{x}} 
    = \sum_{l=1}^{L+1} J(\mathbf{x}',\theta^l)J(\mathbf{x},\theta^l)^{\top} 
    = J(\mathbf{x}',\theta^{L+1})J(\mathbf{x},\theta^{L+1})^{\top},
\end{equation}

where $\theta^l$ indicates the weights at layer $l$. We have,
\begin{equation}
    J(\mathbf{x},W^{L+1})
    = \frac{1}{\sqrt{n_L}}x^L(\mathbf{x}, \theta), \quad
    J(\mathbf{x},b^{L+1})
    = 1.
\end{equation}
Thus,
\begin{equation}
    \hat{\mathbf{\Theta}}_{\mathbf{x}',\mathbf{x}} = \frac{1}{n} x^L(\mathbf{x}',\theta) x^L(\mathbf{x},\theta)^{\top}  + 1.
\end{equation}
This actually converges to the Neural Network Gaussian Process (NNGP) kernel \citep{lee2018dnnsgps}, which does not align with the desired result.

However, our observation highlights that training the network in this configuration corresponds to performing inference in a Gaussian Process (GP) with a zero prior mean and the NNGP kernel as the covariance function. Specifically, in the case of the linearized network (or a network with a sufficiently large layer width), the partial derivatives with respect to the hidden layer parameters are zero. Thus, effectively \emph{only the last layer is trained}. As shown in \citet{lee2019wide}, this scenario is equivalent to performing inference with the NNGP, where the covariance of the trained network matches the posterior covariance of the NNGP.
\newpage
\section{Additional Experiments}\label{app: additional-experiments}

\begin{figure}[h!]  
    \centering
    \begin{subfigure}[b]{0.4\textwidth}
        \centering
        \includegraphics[width=\linewidth]{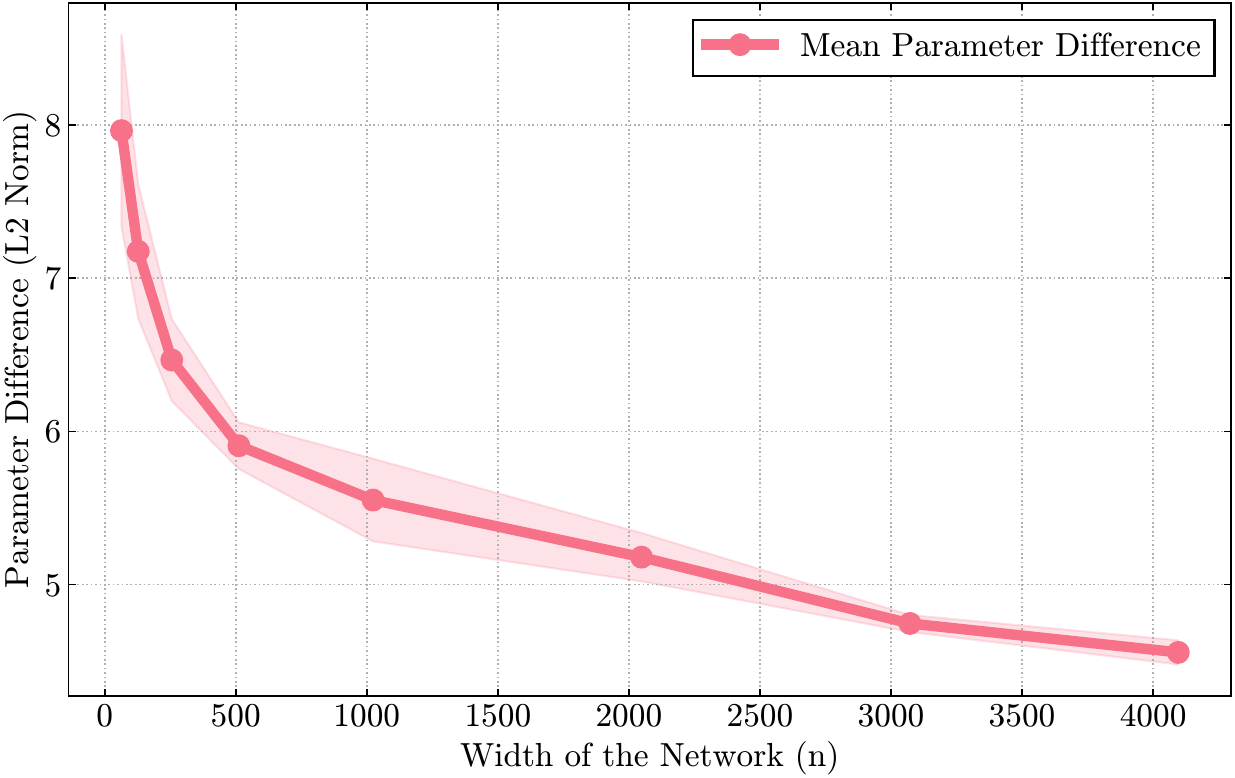}
    \end{subfigure}
    \hspace{0.2in}
    \begin{subfigure}[b]{0.4\textwidth}
        \centering
        \includegraphics[width=\linewidth]{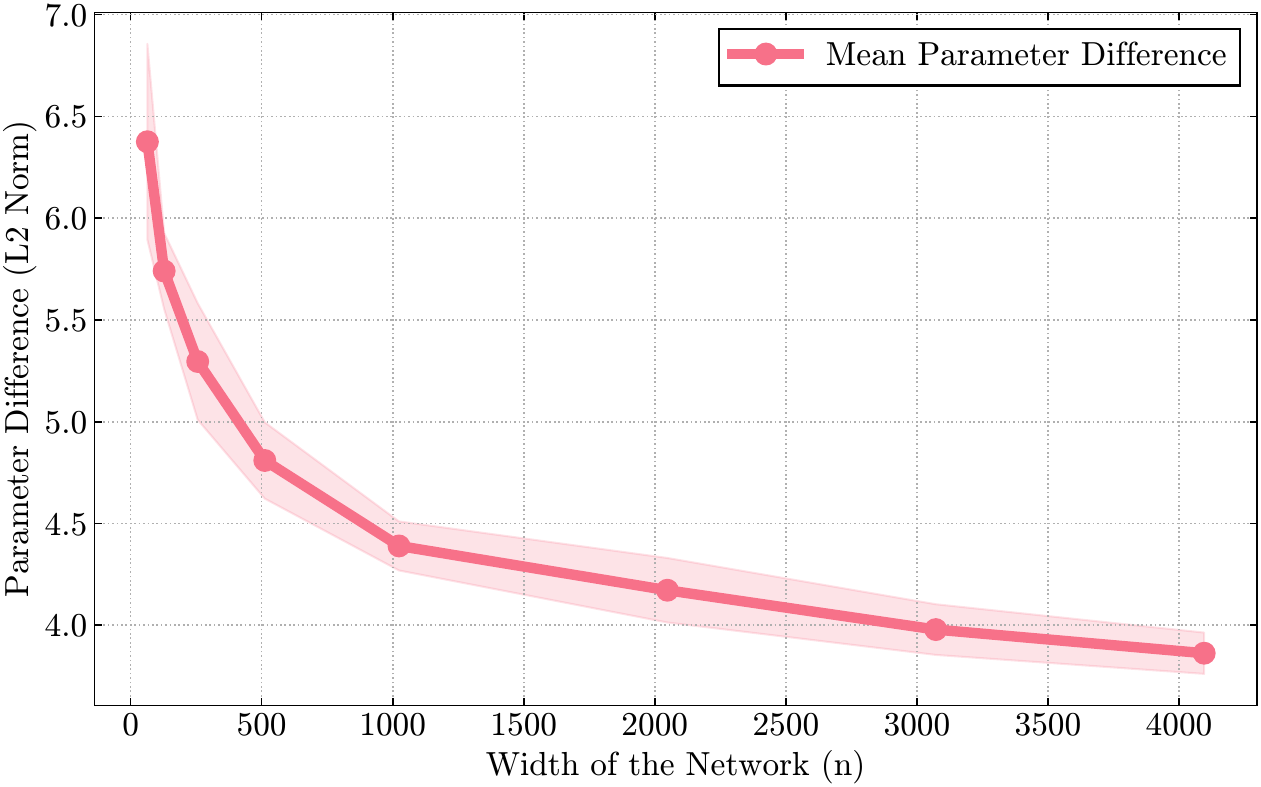}
    \end{subfigure}

    \vspace{0.2in} 

    \begin{subfigure}[b]{0.4\textwidth}
        \centering
        \includegraphics[width=\linewidth]{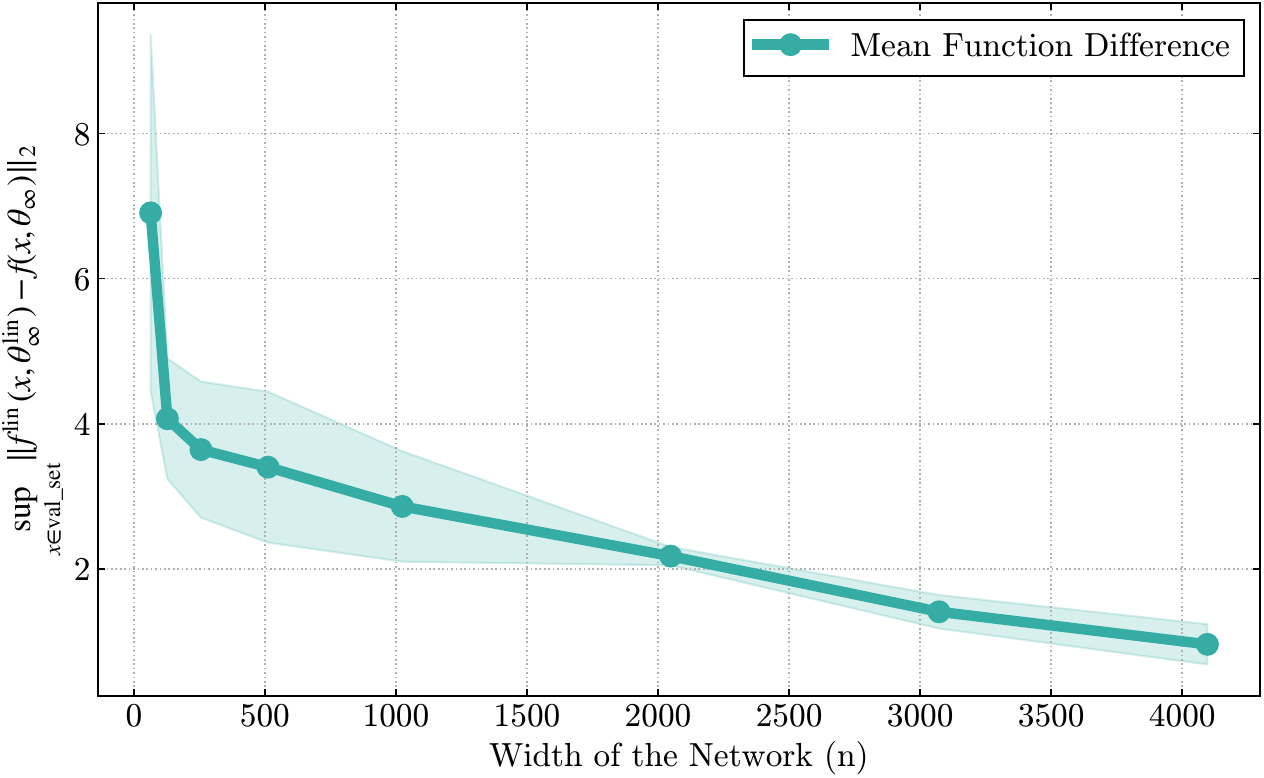}
    \end{subfigure}
    \hspace{0.2in}
    \begin{subfigure}[b]{0.4\textwidth}
        \centering
        \includegraphics[width=\linewidth]{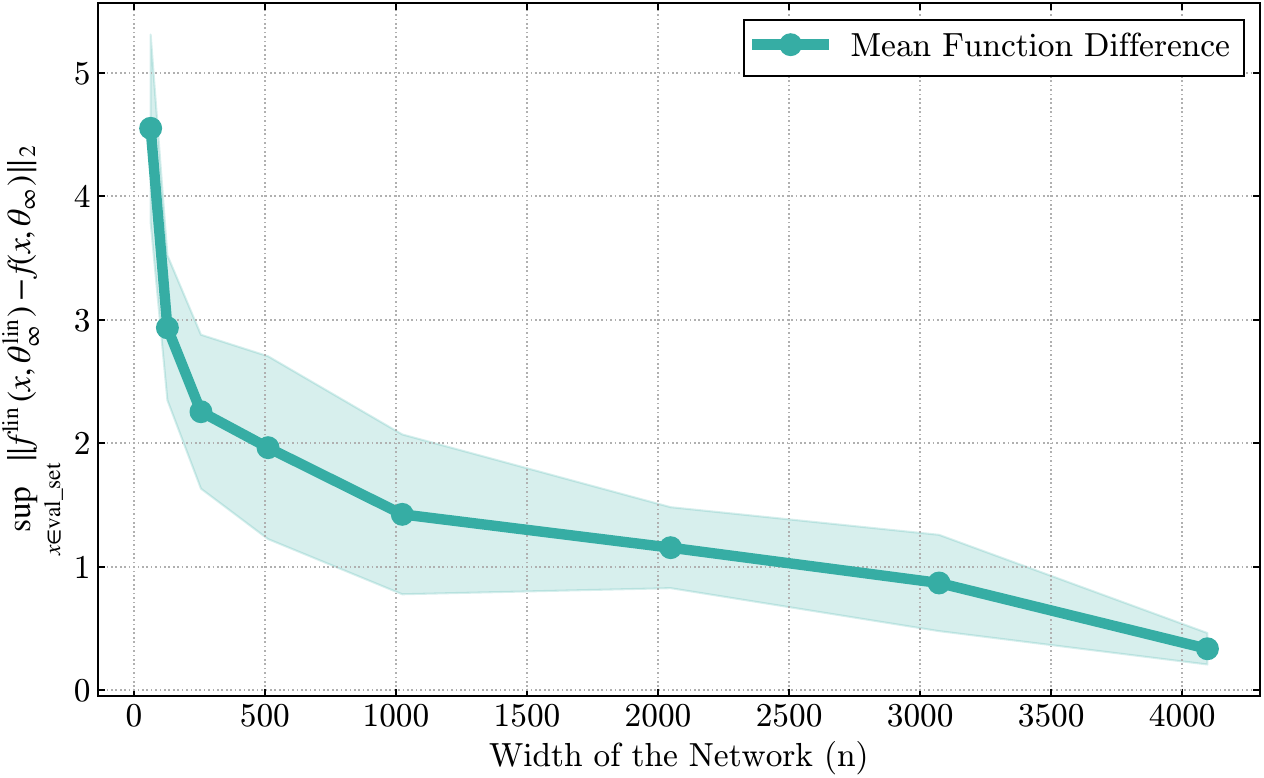}
    \end{subfigure}

    \caption{Parameter and function differences for additional network depths. (Top row) Parameter difference plots from Section \ref{sec: 5.2}. (Bottom row) corresponds to the function difference plots from Section \ref{sec: 5.3}. (Left) Results for one fully connected hidden layer. (Right) Results for an MLP with three fully connected hidden layers. In all cases, increasing the network width reduces both parameter and function differences, confirming the theoretical predictions. $\beta = 0.1$ was used.}
    \label{fig:first_four}
\end{figure}

\begin{figure}[h!]  
    \centering
    \centering
    \includegraphics[width=0.7\linewidth]{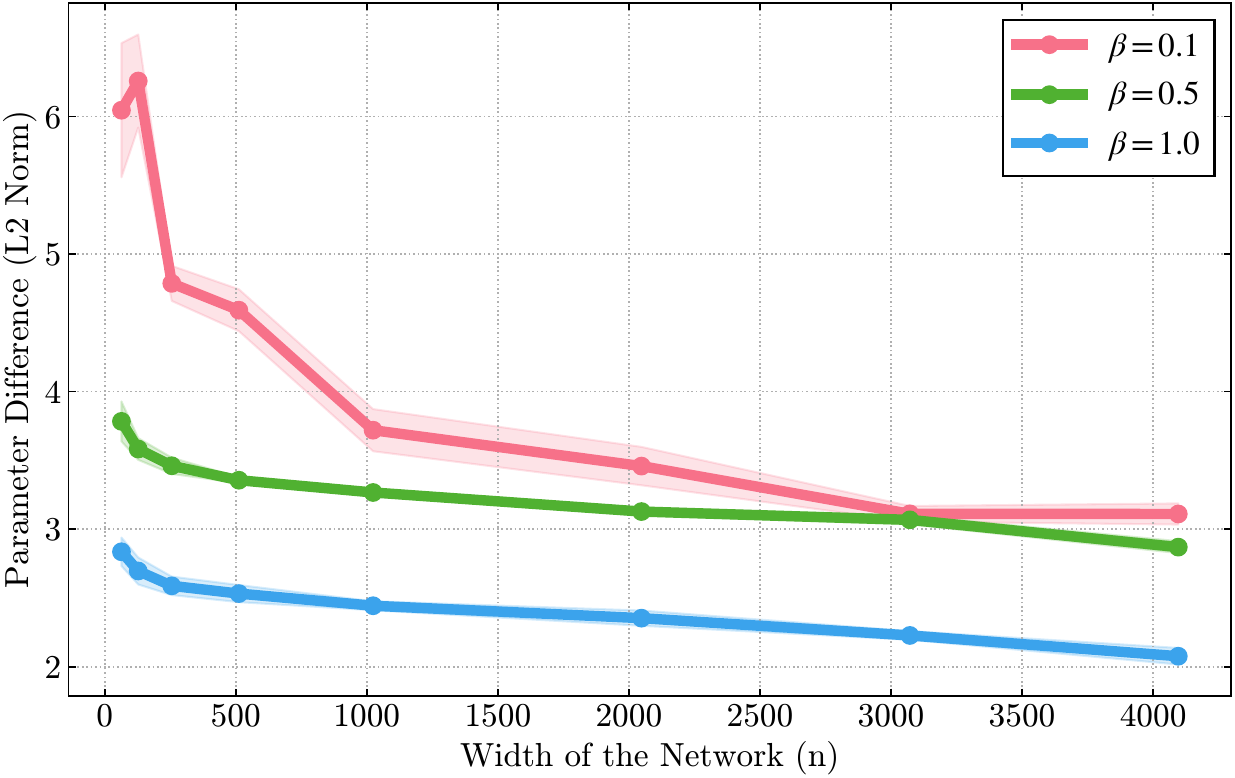}
    \caption{$\ell_2$ norm differences between the trained neural network's parameters and the kernel ridge regression solution plotted against network width for different $\beta$. Shaded regions represent the standard deviation divided by the square root of the number of seeds.}
    \label{fig: params}
\end{figure}

\begin{figure}[h!]  
    \centering
    \includegraphics[width=0.7\linewidth]{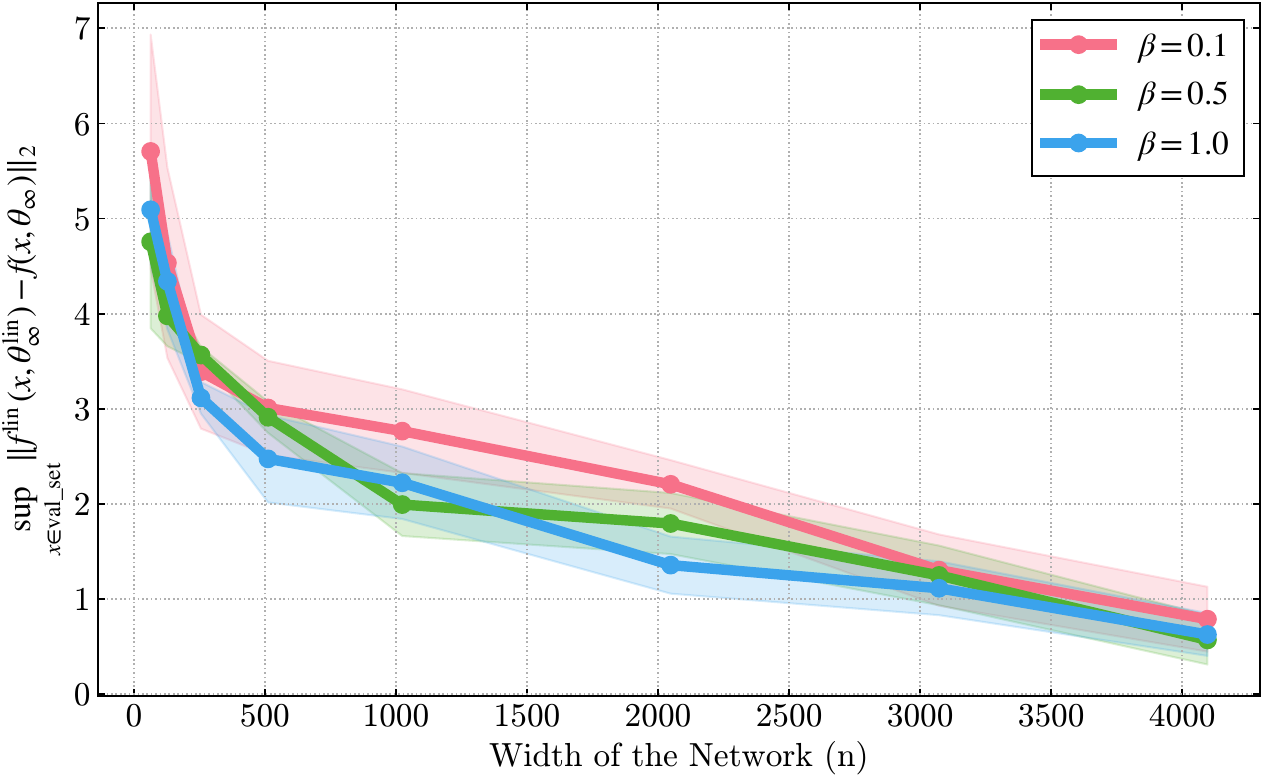}
    \caption{Supremum of the $\ell_2$ norm differences between the trained neural network's outputs \(f(x, \theta_\infty)\) and the linearized model's predictions \(f^{\text{lin}}(x, \theta_\infty^{\text{lin}})\) across the validation set, plotted against network width.\vspace{-0.2in}}
    \label{fig: functions-diff-app}
\end{figure}

\begin{figure}[h!]
    \centering
    \includegraphics[width=0.49\textwidth]{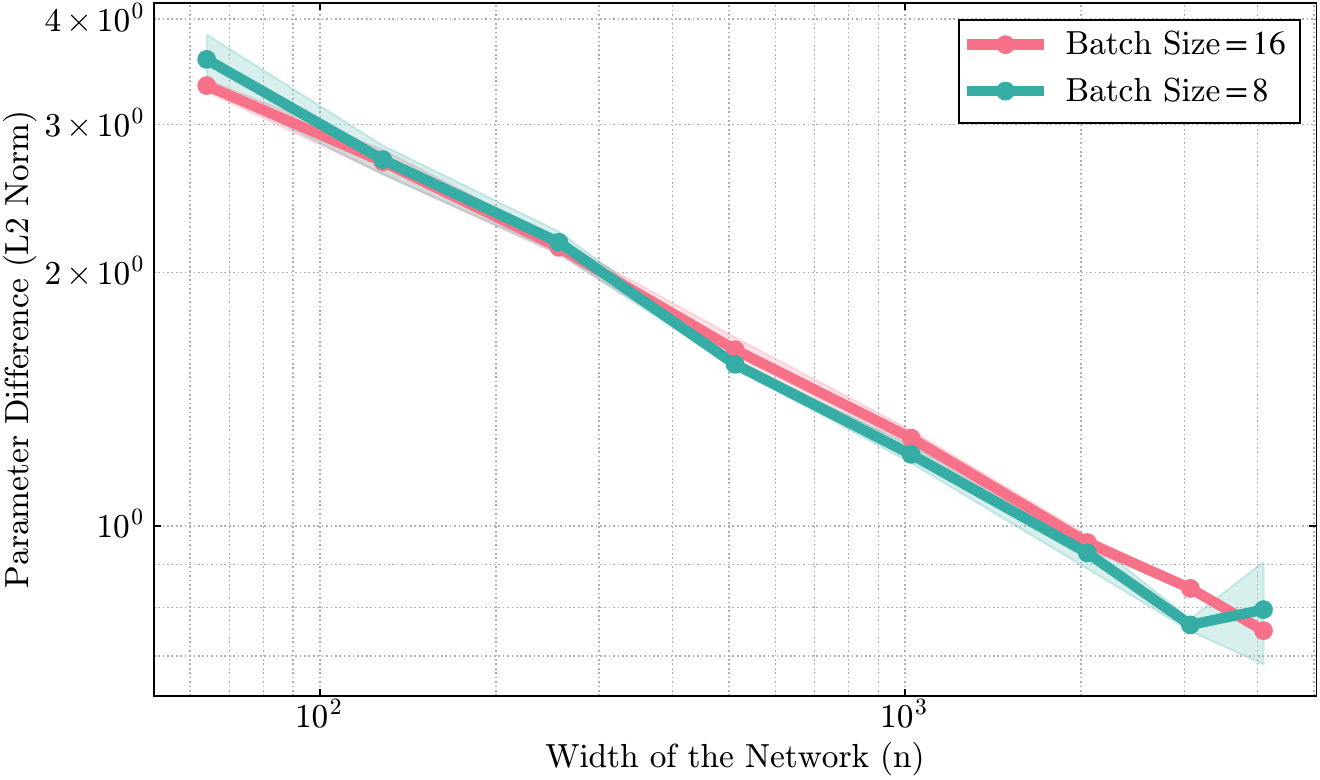}
    \includegraphics[width=0.49\textwidth]{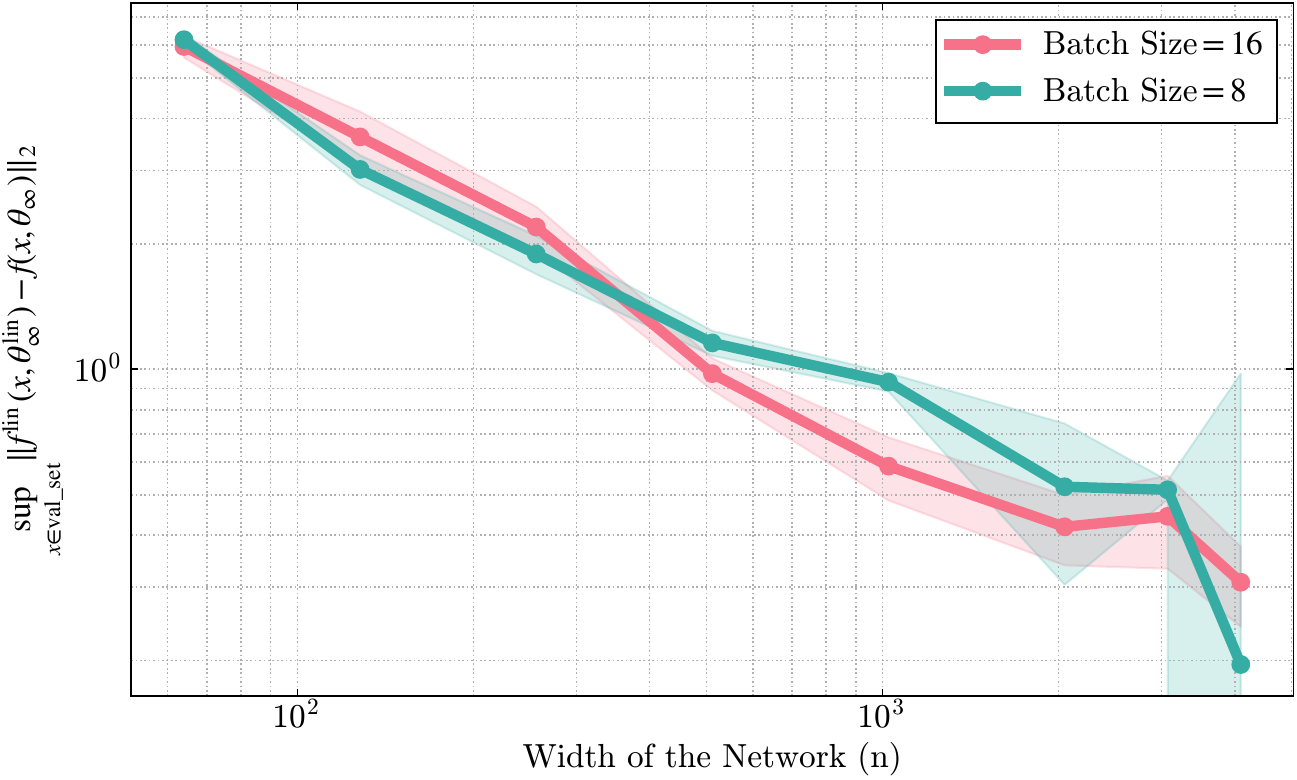}
    \caption{Parameter differences (left) and function differences (right) on the synthetic dataset using mini-batch stochastic gradient descent. Each curve corresponds to a different batch size.}
    \label{fig:SGD-plots}
\end{figure}

\begin{table}[h!]
\centering
\small
\begin{tabular}{l c}
\toprule
\textbf{Method} & \textbf{Test RMSE} $\downarrow$ \\
\midrule
Exact Matern GP & OOM \\
Inducing Points Matern GP (200) & 0.075 \\
NTK-GP via SGD (ours) & 0.099 \\
Ensemble of 5 Wide MLPs & 0.104 \\
Wide MLP (no regularization) & 0.788 \\
Shallow MLP & 0.201 \\
\bottomrule
\end{tabular}
\caption{
Comparison on the Kin8nm regression dataset. We compare our method (NTK-GP posterior mean via SGD with zero prior mean) against exact and approximate Matern GPs, ensembles of wide MLPs \citep{he2020bayesian}, an unregularized wide MLP \citep{lee2019wide}, and a shallow MLP. Exact GP inference resulted in out-of-memory (OOM) issues. Our method outperforms the ensemble while requiring only a single training run and provides a scalable approximation to the GP posterior mean.
}
\label{tab:kin8nm_results}
\end{table}

\clearpage
\newpage
\section{Regularized gradient flow and gradient descent for the linearized network}
\label{appendix: regularized gf/gd for linearized network}
Consider the linearized network $f^{\lin}_{\theta_0}(x, \theta) = f(x,\theta_0) + J(x,\theta_0)(\theta-\theta_0)$. For notational convenience, we drop the dependence on $\theta_0$ throughout the Appendix. In the following, we will consider training the parameters using the regularized training loss 
\begin{equation}
    \mathcal{L}^{\beta, \lin}(\theta) := \frac{1}{2}\sum_{i=1}^N (f^{\lin}(\mathbf{x}_i, \theta) - \mathbf{y}_i)^2 + \frac{1}{2}\beta \lVert \theta - \theta_0 \rVert_2^2.
\end{equation}
\subsection{Regularized gradient flow for the linearized network}
 Then, the evolution of the parameters through gradient flow with learning rate $\eta_0$ is given by
\begin{align}
    \frac{d\theta^{\lin}_t}{dt} 
    &= -\eta_0 \left( J(\theta_0)^{\top} g^{\lin}(\theta^{\lin}_t) + \beta (\theta^{\lin}_t - \theta_0) \right) \\
    &= -\eta_0 \left(J(\theta_0)^{\top} \left( f(\theta_0) + J(\theta_0)(\theta^{\lin}_t - \theta_0) - \mathbf{y} \right) + \beta (\theta^{\lin}_t - \theta_0) \right) \\
    &= -\eta_0 \left(J(\theta_0)^{\top}J(\theta_0) + \beta I_p\right)(\theta^{\lin}_t-\theta_0) - \eta_0J(\theta_0)^{\top}(f(\theta_0)-\mathbf{y}).
\end{align}
This is a multidimensional linear ODE in $\theta^{\lin}_t-\theta_0$. Its unique solution is given by
\begin{align}
    \theta^{\lin}_t 
    &= \theta_0 +  \left( e^{-\eta_0\left(J(\theta_0)^{\top}J(\theta_0) + \beta I_p\right)t} - I_p \right) \left(-\eta_0\left(J(\theta_0)^{\top}J(\theta_0) + \beta I_p \right)\right)^{-1} \left( - \eta_0J(\theta_0)^{\top}(f(\theta_0) - \mathbf{y} ) \right) \\
    &= \theta_0 + \left( I_p - e^{-\eta_0\left(J(\theta_0)^{\top}J(\theta_0) + \beta I_p\right)t} \right) \left(J(\theta_0)^{\top}J(\theta_0) + \beta I_p \right)^{-1} J(\theta_0)^{\top} (\mathbf{y} - f(\theta_0)) \\
    &= \theta_0 + \left( I_p - e^{-\eta_0\left(J(\theta_0)^{\top}J(\theta_0) + \beta I_p\right)t} \right) J(\theta_0)^{\top}\left(J(\theta_0)J(\theta_0)^{\top} + \beta I_N \right)^{-1}  (\mathbf{y} - f(\theta_0)) \\
    &= \theta_0 + \left( J(\theta_0)^{\top} - e^{-\eta_0\left(J(\theta_0)^{\top}J(\theta_0) + \beta I_p\right)t} J(\theta_0)^{\top} \right) \left(J(\theta_0)J(\theta_0)^{\top} + \beta I_N \right)^{-1}  (\mathbf{y} - f(\theta_0)) \\
    &= \theta_0 + J(\theta_0)^{\top} \left(I_N - e^{-\eta_0\left(J(\theta_0)J(\theta_0)^{\top} + \beta I_N \right)t} \right) \left(J(\theta_0)J(\theta_0)^{\top} + \beta I_N \right)^{-1}  (\mathbf{y} - f(\theta_0)).
\end{align}
In the third and fourth equality, we used that for $k\in \mathbb{Z}$,
\begin{equation}
    \left(J(\theta_0)^{\top}J(\theta_0) + \beta I_p \right)^{k} J(\theta_0)^{\top} = J(\theta_0)^{\top}\left(J(\theta_0)J(\theta_0)^{\top} + \beta I_N \right)^{k}.
\end{equation}
Plugging $\theta^{\lin}_t$ into the formula for the linearized network, we get for any point $\mathbf{x}'$,
\begin{align}
    &f^{\lin}(\mathbf{x'},\theta^{\lin}_t) \\
    =& f(\mathbf{x'},\theta_0) + J(\mathbf{x'},\theta_0)J(\theta_0)^{\top} \left( I_N- e^{-\eta_0\left(J(\theta_0)J(\theta_0)^{\top} + \beta I_N\right)t} \right) \left(J(\theta_0)J(\theta_0)^{\top} + \beta I_N \right)^{-1}  (\mathbf{y} - f(\theta_0)) \\
    =&f(\mathbf{x'},\theta_0) + \hat{\mathbf{\Theta}}_{\mathbf{x',x}}\left(I_N - e^{-\eta_0(\hat{\mathbf{\Theta}}_{\mathbf{x,x}} + \beta I_N)t}  \right) \left(\hat{\mathbf{\Theta}}_{\mathbf{x,x}} + \beta I_N \right)^{-1} (\mathbf{y} - f(\theta_0)) .
\end{align}
For training time $t\to\infty$, this gives
\begin{equation}
    f^{\lin}(\mathbf{x'},\theta^{\lin}_\infty)
    = f(\mathbf{x'},\theta_0) + \hat{\mathbf{\Theta}}_{\mathbf{x',x}}\left(\hat{\mathbf{\Theta}}_{\mathbf{x,x}} + \beta I_N \right)^{-1}  (\mathbf{y} - f(\theta_0)).
\end{equation}

\subsection{Regularized gradient descent for the linearized network}
Similarly to gradient flow, the evolution of the parameters through gradient descent (when training the regularized loss given by the linearized network) with learning rate $\eta_0$ is given by
\begin{equation}
    \theta^{\lin}_{t} = \theta^{\lin}_{t-1} -\eta_0 \left(J(\theta_0)^{\top}J(\theta_0) + \beta I_p\right)(\theta^{\lin}_{t-1}-\theta_0) - \eta_0J(\theta_0)^{\top}(f(\theta_0)-\mathbf{y}).
\end{equation}
One may write this as
\begin{equation}
    \theta^{\lin}_{t} - \theta_0 
    = \left(I_p - \eta_0 \left(J(\theta_0)^{\top}J(\theta_0) + \beta I_p\right)\right)(\theta^{\lin}_{t-1}-\theta_0) - \eta_0 J(\theta_0)^{\top} (f(\theta_0) - \mathbf{y}).
\end{equation}
Applying the formula for $\theta_{t}^{\lin}-\theta_0$ iteratively, leads to the following geometric sum:
\begin{align}
    \theta^{\lin}_{t} - \theta_0 
    &= -\eta_0 \sum_{u=0}^{t-1} \left(I_p - \eta_0 \left(J(\theta_0)^{\top}J(\theta_0) + \beta I_p\right)\right)^u J(\theta_0)^{\top}(f(\theta_0) - \mathbf{y}) \\
    &= \eta_0 J(\theta_0)^{\top} \sum_{u=0}^{t-1} \left(I_N - \eta_0 \left(J(\theta_0)J(\theta_0)^{\top} + \beta I_N\right)\right)^u (\mathbf{y} - f(\theta_0)) 
\end{align}
\begin{equation}
    = J(\theta_0)^{\top} \left(I_N - \left(I_N - \eta_0\left(J(\theta_0)J(\theta_0)^{\top} + \beta I_N\right)\right)^{t} \right) \left(J(\theta_0)J(\theta_0)^{\top} + \beta I_N\right)^{-1} (\mathbf{y} - f(\theta_0)).
\end{equation}
This converges for $t\to\infty$ if and only if $0 < \eta_0 < \frac{2}{\lambda_{\text{max}}\left( J(\theta_0)J(\theta_0)^{\top}\right) + \beta}$. In that case, it converges (as expected) to the same $\theta_{\infty}^{\lin}$ as the regularized gradient flow. Plugging $\theta_t^{\lin}$ into the formula for the linearized network, we get for any point $\mathbf{x'}$, 
\begin{align}
    &f^{\lin}(\mathbf{x'},\theta^{\lin}_t) \\
    =& f(\mathbf{x'},\theta_0)  + J(\mathbf{x'},\theta_0)J(\theta_0)^{\top} \left(I_N - \left(I_N - \eta_0\left(J(\theta_0)J(\theta_0)^{\top} + \beta I_N\right)\right)^{t} \right) \left(J(\theta_0)J(\theta_0)^{\top} + \beta I_N\right)^{-1} (\mathbf{y} - f(\theta_0)) \\
    =& f(\mathbf{x'},\theta_0) + \hat{\mathbf{\Theta}}_{\mathbf{x',x}}  \left(I_N - \left( I_N - \eta_0\left(\hat{\mathbf{\Theta}}_{\mathbf{x,x}} + \beta I_N\right)\right)^{t} \right) \left(\hat{\mathbf{\Theta}}_{\mathbf{x,x}} + \beta I_N \right)^{-1} (\mathbf{y} - f(\theta_0)).
\end{align}

\section{Revisiting standard and NTK parametrizations, and convergence at initialization}
\label{appendix: standard and NTK parametrization}
In the following, we revisit the standard and the NTK parametrization. First, we repeat the result about the convergence of the NTK for the NTK parametrization at initialization. Then, we formally state how the NTK and standard parametrization are related, which makes it possible to proof the results for standard parametrization by using the results for NTK parametrization. Finally, we argue that using the same learning rate for every parameter under standard parametrization leads to redundancies in the NTK for the first layer and the biases.

\subsection{Convergence of NTK under NTK parametrization at initialization}
Here, we restate the following Theorem from \citet{yang2020tensorprograms2} about the convergence of the NTK at initialization. This was first shown in \citet{jacot2018neural} when taking the limit of layer widths sequentially. 

\begin{theorem}
\label{theorem. NTK initialization convergence}
Consider a standard feedforward neural network in NTK parametrization. Then, the empirical NTK $\hat{\mathbf{\Theta}}_{\mathbf{x',x}}$ converges to a deterministic matrix $\mathbf{\Theta}_{\mathbf{x',x}}$, which we call the analytical NTK:
\begin{equation}
    \hat{\mathbf{\Theta}}_{x',x} 
    = J(x', \theta_0)J(x, \theta_0)^{\top}
    = \sum_{l=1}^{L+1} \left( J(x', W^l)J(x, W^l)^{\top} + J(x', b^l)J(x, b^l)^{\top} \right)
    \xrightarrow{p} \mathbf{\Theta}_{x', x},
\end{equation}
for layer width $n\to \infty$.
\end{theorem}
Define $\mathbf{\Theta} := \mathbf{\Theta}_{\mathbf{x}, \mathbf{x}} \in \mathbb{R}^{N\times N}$ as the analytical NTK on the training points.
We will assume $\lambda_{\min}(\mathbf{\Theta}) > 0$. A sufficient condition for this is that $\lVert \mathbf{x}_i \rVert_2 = 1 \; \forall i$, and that $\phi$ grows non-polynomially for large $x$, see \citet{jacot2018neural}. This directly implies that for $n$ large enough, the minimum eigenvalue of the analytical NTK is lower bounded by a positive number with high probability:
\begin{lemma}
\label{lemma: empirical NTK positive eigenvalue}
For any $\delta_0 > 0$, there is $n$ large enough, such that with probability of at least $1-\delta_0$, for the minimum eigenvalue of the empirical NTK,
\begin{equation}
    \lambda_{\min}\left(J(\theta_0)J(\theta_0)^{\top} \right) \ge \frac{1}{2}\lambda_{\min}(\mathbf{\Theta}),
    \text{ and } \lambda_{\max}\left(J(\theta_0)J(\theta_0)^{\top} \right) \le 2\lambda_{\max}(\mathbf{\Theta}).
\end{equation}
\end{lemma}

\subsection{Equivalence of NTK parametrization to standard parametrization with layer-dependent learning rates}
\label{appendix: Standard param to NTK subappendix}
In this section, we will formally show how the NTK parametrization relates to the standard parametrization of neural networks. This makes it possible to prove results for standard parametrization by using the results for NTK parametrization, instead of having to prove them again.

Remember that the number of parameters is $p = \sum_{l=1}^{L+1} (n_{l-1} + 1) n_l $.
Define the diagonal matrix $H \in \mathbb{R}^{p\times p}$ through 
\begin{equation}
H := \diag(H_{w,1}, H_{b,1}, \ldots, H_{w,L+1}, H_{b, L+1}),
\end{equation}
where $H_{w,l} := \frac{1}{n_{l-1}} I_{n_{l-1}n_l}$, and $H_{b,l} := I_{n_l}$. The diagonal of $H^{\frac12}$ contains the scalars by which each parameter is multiplied when going from NTK parametrization to standard parametrization. For $\theta^{\std}_0$ initialized in standard parametrization, define
\begin{equation}
    \theta^{\ntk}_0 := H^{-\frac12} \theta^{\std}_0.
\end{equation}
Then, $\theta^{\ntk}_0$ is initialized as in NTK parametrization. Further, let $f^{\std}$ denote a neural network in standard parametrization. Then, 
\begin{equation}
    f^{\ntk}(x, \theta^{\ntk}) := f^{\std}(x, H^{\frac12}\theta^{\ntk}),
\end{equation}
defines a neural network in NTK parametrization. Differentiating gives
\begin{equation}
    J^{\ntk}(x,\theta^{\ntk}) = J^{\std}(x,H^{\frac12}\theta^{\ntk}) H^{\frac12}, \quad
    \hat{\mathbf{\Theta}}_{x',x}^{\ntk} = J^{\std}(x',H^{\frac12}\theta^{\ntk}) H J^{\std}(x,H^{\frac12}\theta^{\ntk}).
\end{equation}

Motivated by this, define the following regularized loss:
\begin{equation}
    \mathcal{L}^{\beta, \std}(\theta) := \frac{1}{2}\sum_{i=1}^N (f^{\std}(\mathbf{x}_i, \theta) - \mathbf{y}_i)^2 + \frac12 \beta (\theta-\theta_0)^{\top}H^{-1}(\theta-\theta_0).
\end{equation}
The resulting gradient is
\begin{equation}
    \nabla_{\theta} \mathcal{L}^{\beta, \std}(\theta) 
    = J^{\std}(\theta)^{\top} g^{\std}(\theta) + \beta H^{-1} (\theta-\theta_0) .
\end{equation}
Define $\theta^{\std}_t$ as parameters evolving by gradient flow of the regularized objective in standard parametrization\footnote{The existence of a unique solution of this ODE will follow from the relation to the gradient flow under NTK parametrization.}, with layer-dependent learning rate $\eta_0H$:
\begin{equation}
    \frac{d\theta^{\std}_t}{dt} 
    = -\eta_0H \nabla_{\theta}\mathcal{L}^{\beta,\std}(\theta^{\std}_t)
    = -\eta_0 H J^{\std}(\theta^{\std}_t)^{\top}g^{\std}(\theta^{\std}_t) - \eta_0 \beta(\theta^{\std}_t - \theta^{\std}_0).
\end{equation}
We define $\theta^{\ntk}_t := H^{-\frac12}\theta^{\std}_t$. Then, as $\frac{d\theta^{\ntk}_t}{dt} = H^{-\frac12}\frac{d\theta^{\std}_t}{dt} $,
\begin{align}
    \frac{d\theta^{\ntk}_t}{dt}
    &= -\eta_0 H^{\frac{1}{2}}J^{\std}(\theta^{\std}_t)^{\top}g^{\std}(\theta^{\std}_t) - \eta_0 H^{-\frac12}\beta (\theta^{\std}_t - \theta^{\std}_0) \\
    &= -\eta_0 \left(J^{\std}(H^{\frac12}\theta^{\ntk}_t) H^{\frac12} \right)^{\top} g^{\std}(H^{\frac12}\theta^{\ntk}_t) - \eta_0\beta (\theta^{\ntk}_t - \theta^{\ntk}_0) \\
    &= -\eta_0 J^{\ntk}(\theta^{\ntk}_t)^{\top} g^{\ntk}(\theta^{\ntk}_t) - \eta_0\beta (\theta^{\ntk}_t - \theta^{\ntk}_0).
\end{align}
Thus, $\theta^{\ntk}_t$ follows the regularized gradient flow of the objective under NTK parametrization with learning rate $\eta_0$. Now, we can apply our results for NTK parametrization from above, and transfer them to standard parametrization by using $\theta^{\std}_t = H^{\frac12}\theta^{\ntk}_t$, $f^{\std}(x,\theta^{\std}_t) = f^{\ntk}(x,H^{-\frac12}\theta^{\std}_t)$.

\subsection{Redundancies when using the same learning rate for standard parametrization}
In the previous section, we established the equivalence between training a neural network in NTK parametrization with learning rate $\eta_0$, and a neural network in standard parametrization with layer-dependent learning rate $\eta_0 H$. By definition, the learning rate for the first layer is $\frac{1}{n_0}\eta_0 = \frac{1}{d}\eta_0$, and the one for the biases is $\eta_0$. The learning rate for the other weight matrices is $\frac{1}{n_{l-1}}\eta_0 = \frac{1}{n}\eta_0$, for $l=2,\ldots,L+1$. Note that the convergence of the learning rates to $0$ for $n\to\infty$ is necessary to stabilize the gradient.

The learning rate that was used in the proof of \citet{lee2019wide} is $\frac{1}{n}\eta_0$ for any layer. In the following, we will argue that this effectively leads to the first layer, and the biases not being trained in the infinite-width limit. For simplicity, let $\beta=0$. Remember that \citet{lee2019wide} shows that using the learning rate $\frac{1}{n}\eta_0$ for each layer in standard parametrization, leads to the trained network for large width being driven by the standard parametrization NTK 
\begin{equation}
    \frac{1}{n}J^{\std}(x',\theta_0)J^{\std}(x,\theta_0)^{\top}
    = \frac{1}{n} \sum_{l=1}^{L+1} \left(J^{\std}(x', W^l_0)J^{\std}(x,W^l_0)^{\top} + J^{\std}(x', b^l_0)J^{\std}(x,b^l_0)^{\top}\right).
\end{equation}
By using the equivalences from the previous section, we may write for $l=2,\ldots,L+1$ (using $H_{w,l} = \frac{1}{n}I_{n_{l-1}n_l}$):
\begin{align}
    \frac{1}{n} J^{\std}(x',W_0^l)J^{\std}(x,W_0^l)^{\top}
    &= \left( J^{\std}(H_{w,l}^{\frac12}\sqrt{n}W_0^l) H_{w,l}^{\frac12}\right) \left( J^{\std}(H_{w,l}^{\frac12}\sqrt{n}W_0^l) H_{w,l}^{\frac12}\right)^{\top} \\
    &= J^{\ntk}(\sqrt{n}W_0^l) J^{\ntk}(\sqrt{n}W_0^l)^{\top}.
\end{align}
This is equal to the empirical NTK under the NTK parametrization for the weights $\sqrt{n}W_{0,i,j}^l \sim \mathcal{N}(0,\sigma_{w,l})$ of the $l$-th layer.
However, for the first layer, we get (using $H_{w,1} = \frac{1}{d}I_{dn}$)
\begin{align}
    \frac{1}{n} J^{\std}(x', W_0^1)J^{\std}(x, W_0^1)^{\top}
    &= \frac{d}{n} \left( J^{\std}(H_{w,1}^{\frac12}\sqrt{d}W_0^1) H_{w,1}^{\frac12}\right) \left( J^{\std}(H_{w,1}^{\frac12}\sqrt{d}W_0^1) H_{w,1}^{\frac12}\right)^{\top} \\
    &= \frac{d}{n} J^{\ntk}(\sqrt{d}W_0^1) J^{\ntk}(\sqrt{d}W_0^1)^{\top} \\
    &\to 0, \text{ for } n\to \infty,
\end{align}
as $J^{\ntk}(\sqrt{d}W_0^1) J^{\ntk}(\sqrt{d}W_0^1)^{\top}$ converges by Theorem \ref{theorem. NTK initialization convergence}.
Similarly for the biases, for $l=1,\ldots,L+1$:
\begin{equation}
    \frac{1}{n}J^{\std}(x',b_0^l)J^{\std}(x,b_0^l)^{\top}
    = \frac{1}{n} J^{\ntk}(x', b_0^l)J^{\ntk}(x,b_0^l)^{\top}
    \to 0, \text{ for } n\to \infty.
\end{equation}
Thus, the analytical standard parametrization NTK of \citet{lee2019wide} doesn't depend on the contribution of the gradient with respect to the first layer and the biases. In other words, using the learning rate $\frac{1}{n}\eta_0$ for the first layer and the biases leads to them not being trained for large widths.

Instead, one may scale the learning rates ``correctly", as motivated by the NTK parametrization in the previous section. For large widths $n$, the trained network is then governed by the following modified NTK for standard parametrization:
\begin{align}
    J^{\std}(\theta)H J^{\std}(\theta)
    =& J^{\std}(W^1)J^{\std}(W^1)^{\top} + J^{\std}(b^1)J^{\std}(b^1)^{\top} \\
    &+ \sum_{l=2}^{L+1} \left( \frac{1}{n}J^{\std}(W^l)J^{\std}(W^l)^{\top} + J^{\std}(b^l)J^{\std}(b^l)^{\top} \right).
\end{align}

\section{Proof for regularized gradient flow}
\label{appendix: Proof regularized gradient flow}
\subsection{Exponential Decay of the Regularized Gradient and Closeness of Parameters to their Initial Value}
\label{appendix: proof regularized gradient flow part 1 (exponential decay)}
\begin{lemma}
\label{lemma: g(theta0) bounded}
Let $\beta\ge 0$. We have for any $t\ge 0$: $\lVert g(\theta_t) \rVert_2 \le \lVert g(\theta_0) \rVert_2$. Further, for any $\delta_0>0$, there is $R_0>0$, such that for $n$ large enough, with probability of at least $1-\delta_0$ over random initialization, $\lVert g(\theta_0) \rVert_2 \le R_0$.
\end{lemma}
\begin{proof}
Using the chain rule and the definition of the gradient flow, we have
\begin{equation}
    \frac{d}{dt} \mathcal{L}^{\beta}(\theta_t)
    = \nabla_{\theta}\mathcal{L}^{\beta}(\theta_t) ^{\top} \frac{d\theta_t}{dt}
    =-\eta_0 \nabla_{\theta}\mathcal{L}^{\beta}(\theta_t) ^{\top} \nabla_{\theta}\mathcal{L}^{\beta}(\theta_t)
    =-\eta_0 \lVert \nabla_{\theta}\mathcal{L}^{\beta}(\theta_t) \rVert_2^2 \le 0.
\end{equation}
Thus, 
\begin{equation}
    \frac12\lVert g(\theta_t) \rVert_2^2 
    \le \frac12\lVert g(\theta_t) \rVert_2^2 + \frac12 \beta \lVert \theta_t - \theta_0 \rVert_2^2
    = \mathcal{L}^{\beta}(\theta_t) 
    \le \mathcal{L}^{\beta}(\theta_0) 
    = \frac{1}{2}\lVert g(\theta_0) \rVert_2^2.
\end{equation}
Hence, $\lVert g(\theta_t) \rVert_2 \le \lVert g(\theta_0) \rVert_2$. Further, note that $f(\theta_0)$ converges in distribution to a Gaussian with mean zero and covariance given by the NNGP kernel \citep{lee2018dnnsgps}. Thus, for $n$ large enough, one can bound $\lVert g(\theta_0) \rVert_2$ with high probability.
\end{proof}
The proof of this Lemma does not apply to the weight-decay regularizer $\frac{1}{2}\lVert \theta \rVert_2^2$ analyzed in \citet{lewkowycz2020training}: In general, $\frac{1}{2}\lVert \theta_0 \rVert_2^2 \neq 0$, and thus we can not use $\mathcal{L}^{\beta}(\theta_0) = \frac{1}{2}\lVert g(\theta_0) \rVert_2^2$. They show that using gradient flow with this regularizer leads to $f_{\theta_{\infty}}(\cdot)=0$ in the infinite-width limit, and we are thus not in the NTK regime.

We restate Theorem \ref{theorem: Exponential decay gradient, parameters stay close} for convenience.
\gftheorempartone*
\begin{proof}
Using Lemma \ref{lemma: g(theta0) bounded}, there is $R_0>0$, such that for $n$ large enough, with probability of at least $1-\frac{1}{3}\delta_0$ over random initialization, $\lVert g(\theta_0) \rVert_2 \le R_0$.
Further, using Lemma \ref{lemma: Jacobian Lipschitz}, let $K$ be the constant for local Lipschitzness/Boundedness of the Jacobian with probability $1-\frac{1}{3}\delta_0$ for $n$ large enough.
Finally, by Lemma \ref{lemma: empirical NTK positive eigenvalue}, for $n$ large enough, with probability of at least $1-\frac{1}{3}\delta_0$ over random initialization, the minimum eigenvalue of the empirical NTK is lower bounded: $\lambda_{\min}(J(\theta_0) J(\theta_0)^{\top}) \ge \frac12 \lambda_{\min}(\mathbf{\Theta})$.\footnote{We will only need this for $\beta = 0$.}
For $n$ large enough, these three events hold with probability of at least $1-\delta_0$ over random initialization. In the following, we consider such initializations $\theta_0$.

Define $c_{\beta} := \frac{1}{2}\beta$ for $\beta > 0$, and $c_{\beta} := \frac{1}{3}\lambda_{\min}(\mathbf{\Theta})$ for $\beta = 0$.\footnote{We note that we could choose any constant smaller than $\beta$ for $\beta > 0$, and similarly, and constant smaller than $\lambda_{\min}(\mathbf{\Theta})$ for $\beta = 0$.} Let $C := \frac{KR_0}{c_{\beta}}$. By Lemma \ref{lemma: Jacobian Lipschitz}, the gradient flow ODE has a unique solution as long as $\theta_t \in B(\theta_0,C)$. Consider $t_1 := \inf\{t\ge 0: \lVert \theta_t - \theta_0 \rVert_2 \ge C \}$. In the following, let $t\le t_1$. Remember that
\begin{equation}
    \frac{d\theta_t}{dt} = -\eta_0 \nabla_{\theta} \mathcal{L}^{\beta}(\theta_t) = -\eta_0 \left(J(\theta_t)^{\top}g(\theta_t) + \beta(\theta_t-\theta_0) \right).
\end{equation}
We want to show that the gradient $\nabla_{\theta} \mathcal{L}^{\beta}(\theta_t)$ of the regularized loss converges to $0$ quickly, and hence $\theta_t$ doesn't move much. For $\beta=0$, its norm is $\lVert J(\theta_t)^{\top}g(\theta_t)\rVert_2$ and hence \citet{lee2019wide} and related proofs showed that $\lVert g(\theta_t) \rVert_2$ converges to $0$ quickly. However, for $\beta>0$, this is not the case, as the training error won't converge to $0$. Instead, we directly look at the dynamics of the norm of the gradient:
\begin{align}
    \frac{d}{dt}\lVert \nabla_{\theta} \mathcal{L}^{\beta}(\theta_t) \rVert_2^2
    &= 2 \left( \nabla_{\theta} \mathcal{L}^{\beta}(\theta_t) \right)^{\top} \nabla_{\theta}^2 \mathcal{L}^{\beta}(\theta_t) \frac{d\theta_t}{dt} \\
    &= -2\eta_0 \left( \nabla_{\theta} \mathcal{L}^{\beta}(\theta_t) \right)^{\top} \nabla_{\theta}^2 \mathcal{L}^{\beta}(\theta_t) \left( \nabla_{\theta} \mathcal{L}^{\beta}(\theta_t) \right).
\end{align}

The Hessian $\nabla_{\theta}^2 \mathcal{L}^{\beta}(\theta_t) \in \mathbb{R}^{p\times p}$ of the regularized loss is given by
\begin{equation}
    \nabla_{\theta}^2 \mathcal{L}^{\beta}(\theta_t)
    = g(\theta_t)^{\top} \nabla_{\theta}^2f(\theta_t) + J(\theta)^{\top}J(\theta_t) + \beta I_p.
\end{equation}
Here, $\nabla_{\theta}^2f(\theta) \in \mathbb{R}^{N\times p\times p}$, and $g(\theta_t)^{\top} \nabla_{\theta}^2f(\theta_t) = \sum_{i=1}^N g(\mathbf{x}_i,\theta_t)\nabla_{\theta}^2f(\mathbf{x}_i,\theta_t)$. 
By using the triangle inequality and Cauchy-Schwarz,
\begin{align}
    \lVert g(\theta_t)^{\top} \nabla_{\theta}^2 f(\theta)  \rVert_2 
    \le \sum_{i=1}^N |g(\mathbf{x}_i, \theta_t)| \lVert \nabla_{\theta}^2f(\mathbf{x}_i,\theta_t) \rVert_2
    \le \lVert g(\theta_t) \rVert_2 \sqrt{\sum_{i=1}^N \lVert \nabla_{\theta}^2f(\mathbf{x}_i,\theta_t) \rVert_2^2}.
\end{align}
Using Lemma \ref{lemma: g(theta0) bounded}, we have $\lVert g(\theta_t) \rVert_2 \le \lVert g(\theta_0) \rVert_2 \le R_0$. Further, by Lemma \ref{lemma: Jacobian Lipschitz}, we have $\lVert \nabla_{\theta}^2f(\mathbf{x}_i,\theta_t) \rVert_2 \le \frac{(\log n)^c}{\sqrt{n}}K'$.
Thus (with $K=\sqrt{N}K'$),
\begin{equation}
    \lVert g(\theta_t)^{\top} \nabla_{\theta}^2 f(\theta)  \rVert_2 \le \frac{(\log n)^c}{\sqrt{n}} K R_0.
\end{equation}
As $g(\theta_t)^{\top} \nabla_{\theta}^2 f(\theta)$ is symmetric, it follows for its minimum eigenvalue, that
\begin{equation}
    \lambda_{\min}\left( g(\theta_t)^{\top} \nabla_{\theta}^2 f(\theta) \right)
    \ge - \lVert g(\theta_t)^{\top} \nabla_{\theta}^2 f(\theta)  \rVert_2 
    \ge - \frac{(\log n)^c}{\sqrt{n}} K R_0.
\end{equation}
Now consider $\beta > 0$. Then, we can follow that for $n$ large enough, the smallest eigenvalue of the Hessian of the regularized loss is positive:
\begin{equation}
    \lambda_{\min}\left(\nabla_{\theta} \mathcal{L}^{\beta}(\theta_t)\right)
    \ge - \frac{(\log n)^c}{\sqrt{n}} K R_0 + 0 + \beta
    \ge \frac{\beta}{2} =: c_{\beta}.
\end{equation}
Thus,
\begin{equation}
    \frac{d}{dt} \lVert \nabla_{\theta} \mathcal{L}^{\beta}(\theta_t) \rVert_2^2
    \le -2\eta_0 c_{\beta} \lVert \nabla_{\theta} \mathcal{L}^{\beta}(\theta_t) \rVert_2^2.
\end{equation}
In Remark \ref{remark: beta=0 proof} we will show, how to modify the proof such that this step is still valid for $\beta = 0$.

By Gronwalls inequality, it follows (for $\beta \ge 0$), that
\begin{equation}
    \lVert \nabla_{\theta} \mathcal{L}^{\beta}(\theta_t) \rVert_2^2
    \le e^{-2\eta_0c_{\beta}t} \lVert \nabla_{\theta} \mathcal{L}^{\beta}(\theta_0) \rVert_2^2.
\end{equation}
Thus,
\begin{equation}
    \lVert \nabla_{\theta} \mathcal{L}^{\beta}(\theta_t) \rVert_2
    \le e^{-\eta_0c_{\beta}t} \lVert \nabla_{\theta} \mathcal{L}^{\beta}(\theta_0) \rVert_2
    \le e^{-\eta_0c_{\beta}t} \lVert J(\theta_0)^{\top}g(\theta_0) \rVert_2
    \le e^{-\eta_0c_{\beta}t} \lVert J(\theta_0) \rVert_2 \lVert g(\theta_0) \rVert_2
    \le KR_0 e^{-\eta_0c_{\beta}t}.
\end{equation}
Hence, for the distance of parameters from initialization
\begin{equation}
    \lVert \theta_t - \theta_0 \rVert_2
    = \norm{ \int_0^t \frac{d\theta_u}{du} du }_2
    \le \int_0^t \norm{ \frac{d\theta_u}{du} }_2 du
    \le \eta_0 K R_0 \int_0^t e^{-\eta_0 c_{\beta}u}du
    = \frac{KR_0}{c_{\beta}} (1- e^{-\eta_0 c_{\beta}t}).
\end{equation}
Thus, for $t \le t_1$, $\lVert \theta_t - \theta_0 \rVert_2 < C$. By continuity, $t_1 = \infty$.
Using local Lipschitzness, we can further bound the distance of the Jacobian at any $\lVert x \rVert_2 \le 1$:
\begin{equation}
    \lVert J(x,\theta_t) - J(x, \theta_0) \rVert_2
    \le \frac{(\log n)^c}{\sqrt{n}}K' \lVert \theta_t - \theta_0 \rVert_2
    \le \frac{(\log n)^c}{\sqrt{n}}K'C.
\end{equation}
This finishes the proof of Theorem \ref{theorem: Exponential decay gradient, parameters stay close}.
\end{proof}
\begin{remark}
\label{remark: beta=0 proof}
For $\beta = 0$, the Hessian is 
\begin{equation}
    \nabla_{\theta}^2 \mathcal{L}^{0}(\theta_t) 
    = g(\theta_t)^{\top} \nabla_{\theta}^2 f(\theta_t) + J(\theta_t)^{\top}J(\theta_t).
\end{equation}
For $\beta>0$, we just used that $J(\theta_t)^{\top}J(\theta_t) $ is positive semi-definite, as $\beta I$ dominates the negative eigenvalues of the first term of the Hessian. For $\beta = 0$, this is not enough. $J(\theta_t)^{\top}J(\theta_t) \in \mathbb{R}^{p\times p}$ shouldn't be confused with the NTK $J(\theta_t)J(\theta_t)^{\top} \in \mathbb{R}^{N\times N}$. However, they share the have the same nonzero eigenvalues. For $p>N$ (which is the case for $n$ large enough), $J(\theta_t)^{\top}J(\theta_t)$ will additionally have the eigenvalue $0$ with multiplicity of at least $p-N$. Thus, we can't naively lower bound the minimum eigenvalue of the Hessian with the minimum eigenvalue of $J(\theta_t)^{\top}J(\theta_t)$. 

Luckily, $\nabla_{\theta} \mathcal{L}^0(\theta_t) = J(\theta_t)^{\top}g(\theta_t)$ is in the row-span of $J(\theta_t)$. This is orthogonal to the nullspace of $J(\theta_t)$, i.e. the eigenspace corresponding to the eigenvalue $0$ of $J(\theta_t)$. Thus, $\nabla_{\theta} \mathcal{L}^0(\theta_t)$ only ``uses" the positive eigenvalues of $J(\theta_t)^{\top}J(\theta_t)$. The smallest positive eigenvalue of $J(\theta_t)^{\top}J(\theta_t)$ is equal to the smallest positive eigenvalue of the empirical NTK $J(\theta_t)J(\theta_t)^{\top}$, which is lower bounded by $\frac{1}{2}\lambda_{\min}(\mathbf{\Theta})$ on the high probability event we consider.

Hence, for $n$ large enough,
\begin{align}
    \frac{d}{dt} \lVert \nabla_{\theta} \mathcal{L}^0(\theta_t) \rVert_2^2
    &= -2\eta_0 \left( \nabla_{\theta} \mathcal{L}^0(\theta_t) \right)^{\top} \nabla_{\theta}^2 \mathcal{L}^0(\theta_t) \left( \nabla_{\theta} \mathcal{L}^0(\theta_t) \right) \\
    &\le -2\eta_0 \left(\lambda_{\min}\left( g(\theta_t)^{\top} \nabla_{\theta}^2 f(\theta) \right)  + \frac{1}{2}\lambda_{\min}(\mathbf{\Theta}) \right)\lVert \nabla_{\theta} \mathcal{L}^0(\theta_t) \rVert_2^2 \\
    &\le -2\eta_0 \left( -\frac{(\log n)^c}{\sqrt{n}}KR_0 + \frac{1}{2}\lambda_{\min}(\mathbf{\Theta}) \right)\lVert \nabla_{\theta} \mathcal{L}^0(\theta_t) \rVert_2^2 \\
    &\le -2\eta_0 \frac{1}{3}\lambda_{\min}(\mathbf{\Theta}) \lVert \nabla_{\theta} \mathcal{L}^0(\theta_t) \rVert_2^2.
\end{align}
Defining $c_0 := \frac{1}{3}\lambda_{\min}(\mathbf{\Theta})$ makes it possible to continue with the proof above for $\beta\ge 0$. This is an alternative proof to \citet{lee2019wide}. It shows that in the unregularized case, it is important that the gradient flow lies in the row space of the Jacobian.
\end{remark}

\subsection{Closeness to the Linearized Network along the Regularized Gradient Flow}
\label{appendix: Proof regularized gradient flow part 2 (closeness linearized)}
The goal of this section is to prove that the neural network along the regularized gradient flow stays close to the linearized network along the linear regularized gradient flow. Let us restate the following Theorem.
\gftheoremparttwo*
\begin{proof}
The proof of \citet{lee2019wide} used that training error converges to $0$. Thus, we have to use a different approach, which also provides a more straightforward and intuitive proof for $\beta=0$. Remember that
\begin{equation}
    f^{\lin}(x,\theta) = f(x,\theta_0) + J(x,\theta_0)(\theta-\theta_0), \text{ and } 
    \frac{d\theta^{\text{lin}}}{dt} = -\eta_0 \left(J(\theta_0)^{\top} g^{\text{lin}}(\theta^{\text{lin}}_t) + \beta (\theta^{\text{lin}}_t - \theta_0)\right).
\end{equation}
To prove the second part of the Theorem, we will use
\begin{equation}
\label{eq: triangle f flin proof}
    \lVert f(x,\theta_t) - f^{\lin}(x,\theta^{\lin}_t)  \rVert_2
    \le \lVert f(x,\theta_t) - f^{\lin}(x,\theta_t) \rVert_2 + \lVert f^{\lin}(x,\theta_t) - f^{\lin}(x,\theta^{\lin}_t)\rVert_2.
\end{equation}
We will start by bounding the first term. Next, we will bound $\lVert \theta_t - \theta^{\lin}_t \rVert_2$, and use this to bound the second term.

\textbf{First step:} To bound $\lVert f(x,\theta_t) - f^{\lin}(x,\theta_t) \rVert_2$, we compute
\begin{align}
    \norm{ \frac{d}{dt} \left(f(x,\theta_t) - f^{\lin}(x,\theta_t) \right) }_2
    &= \norm{\left(J(x,\theta_t) - J(x,\theta_0) \right)\frac{d\theta_t}{dt}}_2 \\
    &\le \lVert J(x,\theta_t) - J(x,\theta_0) \rVert_2 \norm{\frac{d\theta_t}{dt}}_2 \\
    &\le \frac{(\log n)^c}{\sqrt{n}} K' C \eta_0 KR_0 e^{-\eta_0 c_{\beta}t},
\end{align}
where we used Theorem \ref{theorem: Exponential decay gradient, parameters stay close} in the last step.
Now, we can bound
\begin{equation}
    \lVert f(x,\theta_t) - f^{\lin}(x,\theta_t) \rVert_2
    \le \frac{(\log n)^c}{\sqrt{n}}K'C\eta_0KR_0 \int_0^t e^{-\eta_0 c_{\beta} u} du
    \le \frac{(\log n)^c}{\sqrt{n}}K'C \frac{KR_0}{c_{\beta}} 
    = \frac{(\log n)^c}{\sqrt{n}} K' C^2.
\end{equation}
In particular, for the difference at the training points, $\lVert f(\theta_t) - f^{\lin}(\theta_t) \rVert_2 \le \frac{(\log n)^c}{\sqrt{n}} K C^2$.

\textbf{Second step:} Now, we will bound the difference between $\theta_t - \theta^{\lin}_t$. We can write
\begin{align}
    \frac{d\theta_t}{dt}
    &= -\eta_0 \left( J(\theta_t)^{\top}g(\theta_t) + \beta(\theta_t-\theta_0) \right) \\
    &= -\eta_0 \left( \left( J(\theta_t) - J(\theta_0) \right)^{\top}g(\theta_t) + J(\theta_0)^{\top} \left( g(\theta_t) - g^{\lin}(\theta_t) \right) + J(\theta_0)^{\top}g^{\lin}(\theta_t) + \beta(\theta_t-\theta_0) \right) \\
    &= -\eta_0 \Delta_t - \eta_0 \left( J(\theta_0)^{\top}g^{\lin}(\theta_t) + \beta(\theta_t-\theta_0) \right),
\end{align}
where we define $\Delta_t := \left( J(\theta_t) - J(\theta_0) \right)^{\top}g(\theta_t) + J(\theta_0)^{\top} \left( g(\theta_t) - g^{\lin}(\theta_t) \right)$. We will now bound $\lVert \Delta_t \rVert_2$. For the first term, using Theorem \ref{theorem: Exponential decay gradient, parameters stay close}:
\begin{equation}
    \lVert \left( J(\theta_t) - J(\theta_0) \right)^{\top}g(\theta_t) \rVert_2
    \le \lVert J(\theta_t) -J(\theta_0)\rVert_2 \lVert g(\theta_t) \rVert_2
    \le \frac{(\log n)^c}{\sqrt{n}}KCR_0.
\end{equation}
For the second term, using Theorem \ref{theorem: Exponential decay gradient, parameters stay close} and the bound we derived in the first step, 
\begin{equation}
    \lVert J(\theta_0)^{\top}\left( g(\theta_t) - g^{\lin}(\theta_t) \right) \rVert_2
    = \lVert J(\theta_0)^{\top}\left( f(\theta_t) - f^{\lin}(\theta_t) \right) \rVert_2
    \le \lVert J(\theta_0)\rVert_2 \lVert f(\theta_t) - f^{\lin}(\theta_t) \rVert_2
    \le \frac{(\log n)^c}{\sqrt{n}}K^2C^2.
\end{equation}
Thus, defining $K^{\Delta} := KCR_0 + K^2C^2$, we can bound $\lVert \Delta_t \rVert_2 \le \frac{(\log n)^c}{\sqrt{n}} K^{\Delta}$. Now, we can compute

\begin{align}
    \frac{d}{dt} (\theta_t - \theta^{\lin}_t)
    &= -\eta_0 \Delta_t - \eta_0 \left(J(\theta_0)^{\top} g^{\lin}(\theta_t) + \beta(\theta_t - \theta_0) \right) + \eta_0 \left(J(\theta_0)^{\top} g^{\lin}(\theta^{\lin}_t) + \beta(\theta^{\lin}_t - \theta_0) \right) \\
    &= -\eta_0\Delta_t - \eta_0 \left(J(\theta_0)^{\top} \left(g^{\lin}(\theta_t) - g^{\lin}(\theta^{\lin}_t)\right) + \beta(\theta_t - \theta^{\lin}_t) \right) \\
    &= -\eta_0\Delta_t - \eta_0 \left(J(\theta_0)^{\top} J(\theta_0)(\theta_t - \theta^{\lin}_t) + \beta(\theta_t - \theta^{\lin}_t) \right) \\
    &= -\eta_0\Delta_t - \eta_0 \left(J(\theta_0)J(\theta_0)^{\top} + \beta I \right)(\theta_t - \theta^{\lin}_t).
\end{align}
By treating this as an inhomogeneous linear ODE in $\theta_t - \theta^{\lin}_t$, we get
\begin{equation}
    \theta_t - \theta^{\lin}_t = \int_0^t e^{-\eta_0 \left(J(\theta_0)J(\theta_0)^{\top} + \beta I \right) (t-u)} (-\eta_0 \Delta_u) du.
\end{equation}
Hence (using $\lVert e^{-A} \rVert_2 \le e^{-\lambda_{\min}(A)}$),
\begin{align}
    \lVert \theta_t - \theta^{\lin}_t \rVert_2
    &\le \int_0^t \lVert e^{-\eta_0 \left(J(\theta_0)J(\theta_0)^{\top} + \beta I \right) (t-u)} \rVert_2 \eta_0 \lVert \Delta_u \rVert_2 du \\
    &\le \int_0^{t} e^{-\eta_0 (c_0 + \beta)(t-u)} \eta_0 \frac{(\log n)^c}{\sqrt{n}} K^{\Delta} du \\
    &\le \frac{(\log n)^c}{\sqrt{n}} \frac{K^{\Delta}}{c_0 + \beta}.
\end{align}
Thus, $\sup_t \lVert \theta_t - \theta^{\lin}_t \rVert_2 \le \frac{K^{\Delta}}{c_0 + \beta} \frac{(\log n)^c}{\sqrt{n}}$.

\textbf{Third step:} Using the bound on $\lVert \theta_t - \theta^{\lin}_t \rVert_2$, we can easily bound $\lVert f^{\lin}(x,\theta_t) - f^{\lin}(x,\theta^{\lin}_t)\rVert_2$:
\begin{equation}
    \lVert f^{\lin}(x,\theta_t) - f^{\lin}(x,\theta^{\lin}_t)\rVert_2
    = \lVert J(x,\theta_0) (\theta_t - \theta^{\lin}_t) \rVert_2
    \le \lVert J(x,\theta_0) \rVert_2 \lVert \theta_t - \theta^{\lin}_t \rVert_2
    \le K' \frac{K^{\Delta}}{c_0 + \beta} \frac{(\log n)^c}{\sqrt{n}}.
\end{equation}
By using equation (\ref{eq: triangle f flin proof}), we can now finish the proof:
\begin{equation}
    \lVert f(x,\theta_t) - f^{\lin}(x,\theta^{\lin}_t) \rVert_2
    \le \left(K'C^2 + K' \frac{K^{\Delta}}{c_0 + \beta} \right) \frac{(\log n)^c}{\sqrt{n}}.
\end{equation}
\end{proof}

\section{Proof for regularized gradient descent}
\label{appendix: gradient descent}
\subsection{Geometric Decay of the regularized gradient and closeness of parameters to their initial value}
\begin{theorem}
\label{theorem: gradient descent geometric decay}
Let $\beta \ge 0$. Let $\delta_0 > 0$ arbitrarily small. There are $K',K,R_0, c_{\beta}, \eta_{\max} > 0$, such that for $n$ large enough, the following holds with probability of at least $1-\delta_0$ over random initialization, when applying regularized gradient descent with learning rate $\eta = \eta_0 \le \eta_{\max}$:
\begin{equation}
    \lVert \theta_{t+1} - \theta_{t}\rVert_2  
    = \eta_0 \lVert \nabla_{\theta} \mathcal{L}^{\beta}(\theta_{t}) \rVert_2 
    \le \eta_0KR_0 \left( 1 - \eta_0c_{\beta} \right)^{t},
\end{equation}
\begin{equation}
    \lVert \theta_t - \theta_0 \rVert_2
    <  \frac{KR_0}{c_{\beta}}
    =: C,
\end{equation}
\begin{equation}
    \forall \lVert x \rVert_2 \le 1: \lVert J(x,\theta_t) - J(x,\theta_0) \rVert_2 \le \frac{(\log n)^c}{\sqrt{n}} K' C,
\end{equation}
\begin{equation}
    \lVert J(\theta_t) - J(\theta_0) \rVert_2 \le \frac{(\log n)^c}{\sqrt{n}} K C.
\end{equation}
\end{theorem}
\begin{proof}
We consider the same high probability event as in the proof for the regularized gradient flow. 
Define $c_{\beta} := \frac{1}{2}\beta$ for $\beta > 0$, and $c_{\beta} := \frac{1}{3}\lambda_{\min}(\mathbf{\Theta})$ for $\beta=0$. Let $C:=\frac{KR_0}{c_{\beta}}$.

We will prove the first two inequalities by induction. For $t=0$:
\begin{equation}
    \lVert \nabla_{\theta} \mathcal{L}^{\beta}(\theta_t) \rVert_2 
    = \lVert J(\theta_0)^{\top}g(\theta_0) \rVert_2 
    \le KR_0.
\end{equation}
Now, assume it holds true for $s\le t$.
We want to bound $\lVert \theta_{t+1} - \theta_t \rVert_2 = \eta_0 \lVert \nabla_{\theta} \mathcal{L}^{\beta}(\theta_t)\rVert_2$.
Remember that $\nabla_{\theta} \mathcal{L}^{\beta}(\theta_t) = J(\theta_t)^{\top}g(\theta_t) + \beta (\theta_t - \theta_0)$.
We write
\begin{align}
    \lVert \nabla_{\theta} \mathcal{L}^{\beta}(\theta_t) \rVert_2 ^2
    = &\lVert \nabla_{\theta} \mathcal{L}^{\beta}(\theta_{t-1}) + \nabla_{\theta} \mathcal{L}^{\beta}(\theta_{t}) - \nabla_{\theta} \mathcal{L}^{\beta}(\theta_{t-1}) \rVert_2^2 \\
    =& \lVert \nabla_{\theta} \mathcal{L}^{\beta}(\theta_{t-1}) \rVert_2^2 \\
    &+ 2 \nabla_{\theta}\mathcal{L}^{\beta}(\theta_{t-1})^{\top} \left( \nabla_{\theta} \mathcal{L}^{\beta}(\theta_{t}) - \nabla_{\theta} \mathcal{L}^{\beta}(\theta_{t-1}) \right) \label{eq: gd second term} \\
    &+ \lVert \nabla_{\theta} \mathcal{L}^{\beta}(\theta_{t}) - \nabla_{\theta} \mathcal{L}^{\beta}(\theta_{t-1}) \rVert_2^2. \label{eq: gd third term}
\end{align}
In the following, we will look at how to bound the second and the third term. We have:
\begin{align}
    \nabla_{\theta} \mathcal{L}^{\beta}(\theta_{t}) - \nabla_{\theta} \mathcal{L}^{\beta}(\theta_{t-1}) 
    &= \nabla_{\theta} \mathcal{L}^{\beta}(\theta_{t-1} - \eta_0 \nabla_{\theta}\mathcal{L}^{\beta}(\theta_{t-1})) - \nabla_{\theta} \mathcal{L}^{\beta}(\theta_{t-1}) \\
    &= -\int_0^{\eta_0}  \left(\nabla_{\theta}^2 \mathcal{L}^{\beta}\left( \theta_{t-1} - u\nabla_{\theta}\mathcal{L}^{\beta}(\theta_{t-1}) \right) \right) \cdot \nabla_{\theta}\mathcal{L}^{\beta}(\theta_{t-1}) du.
\end{align}
As in the proof for the gradient flow, the following part only holds for $\beta > 0$.
Note, that for any $u\in [0,\eta_0]$, $\theta_{t-1} - u\nabla_{\theta} \mathcal{L}^{\beta}(\theta_{t-1}) \in B(\theta_0, C)$, as we know by induction that $\theta_{t-1},\theta_t \in B(\theta_0,C)$. Thus, similar to the proof for the gradient flow, for $n$ large enough, $\forall u \in [0,\eta_0]$:
\begin{equation}
    \lambda_{\min}\left( \nabla_{\theta}^2\mathcal{L}^{\beta}\left( \theta_{t-1} - u\nabla_{\theta} \mathcal{L}^{\beta}(\theta_{t-1})\right) \right) \ge \frac{3}{2}c_{\beta}.
\end{equation}
Note that we are using $\frac{3}{2}c_{\beta} = \frac{3}{4}\beta$, which is slightly higher than $c_{\beta}$ which we used in the gradient flow case, to arrive at the equivalent result in the end.
For the second term (\ref{eq: gd second term}) we get,

\begin{align}
    &2 \nabla_{\theta}\mathcal{L}^{\beta}(\theta_{t-1})^{\top} \left( \nabla_{\theta} \mathcal{L}^{\beta}(\theta_{t}) - \nabla_{\theta} \mathcal{L}^{\beta}(\theta_{t-1}) \right) \\
    =& - 2\int_0^{\eta_0} \nabla_{\theta}\mathcal{L}^{\beta}(\theta_{t-1}) \left(\nabla_{\theta}^2 \mathcal{L}^{\beta}\left( \theta_{t-1} - u\nabla_{\theta}\mathcal{L}^{\beta}(\theta_{t-1}) \right) \right) \cdot \nabla_{\theta}\mathcal{L}^{\beta}(\theta_{t-1})du \\
    \le& -2\eta_0 \frac{3}{2}c_{\beta} \lVert \nabla_{\theta} \mathcal{L}^{\beta}(\theta_{t-1}) \rVert_2^2.
\end{align}
Further, we have for any $\theta \in B(\theta_0, C)$:
\begin{align}
    \lVert \nabla_{\theta}^2 \mathcal{L}^{\beta}\left( \theta\right) \rVert_2
    &\le \lVert g(\theta)^{\top} \nabla_{\theta}^2 f(\theta)\rVert_2 + \lVert J(\theta)^{\top}J(\theta) \rVert_2 + \beta I_p \\
    &\le \frac{(\log n)^c}{\sqrt{n}} KR_0 + \lambda_{\max}\left(J(\theta)^{\top}J(\theta)\right) + \beta \\
    &\le \frac{(\log n)^c}{\sqrt{n}} KR_0 + 2\lambda_{\max}(\mathbf{\Theta}) + \beta \\
    &\le 2(\lambda_{\max}(\mathbf{\Theta}) + \beta),
\end{align}
for $n$ large enough.
Using this with $\theta = \theta_{t-1} - u\nabla_{\theta} \mathcal{L}^{\beta}(\theta_{t-1})$, we get for the third term (\ref{eq: gd third term}),
\begin{align}
    \lVert \nabla_{\theta} \mathcal{L}^{\beta}(\theta_{t}) - \nabla_{\theta} \mathcal{L}^{\beta}(\theta_{t-1}) \rVert_2^2
    &= \norm{\int_0^{\eta_0}  \left(\nabla_{\theta}^2 \mathcal{L}^{\beta}\left( \theta_{t-1} - u\nabla_{\theta}\mathcal{L}^{\beta}(\theta_{t-1}) \right) \right) \cdot \nabla_{\theta}\mathcal{L}^{\beta}(\theta_{t-1}) du}_2^2 \\
    &\le \left( \int_0^{\eta_0} \lVert \nabla_{\theta}^2 \mathcal{L}^{\beta}\left( \theta_{t-1} - u\nabla_{\theta}\mathcal{L}^{\beta}(\theta_{t-1}) \right) \rVert_2 \lVert \nabla_{\theta}\mathcal{L}^{\beta}(\theta_{t-1}) \rVert_2 du \right)^2 \\
    &\le \eta_0^2 \left(2(\lambda_{\max}(\mathbf{\Theta}) + \beta) \right)^2 \lVert \nabla_{\theta}\mathcal{L}^{\beta}(\theta_{t-1}) \rVert_2^2 \\
    &\le \eta_0 c_{\beta} \lVert \nabla_{\theta}\mathcal{L}^{\beta}(\theta_{t-1}) \rVert_2^2.
\end{align}
In the last inequality, we chose the learning rate $\eta_0 \le \frac{c_{\beta}}{4(\lambda_{\max}(\mathbf{\Theta}) + \beta)^2}$ small enough. Similarly to the gradient flow case, one can derive such bounds for $\beta=0$.
Summing up the three terms,
\begin{equation}
    \lVert \nabla_{\theta}\mathcal{L}^{\beta}(\theta_t) \rVert_2^2 
    \le (1 - 2\eta_0\frac{3}{2}c_{\beta} + \eta_0c_{\beta})\lVert \nabla_{\theta}\mathcal{L}^{\beta}(\theta_{t-1}) \rVert_2^2
    = (1-2\eta_0 c_{\beta})\lVert \nabla_{\theta}\mathcal{L}^{\beta}(\theta_{t-1}) \rVert_2^2.
\end{equation}
Thus, by Bernoulli's inequality and the induction hypothesis,
\begin{equation}
    \lVert \nabla_{\theta}\mathcal{L}^{\beta}(\theta_t) \rVert_2
    \le \sqrt{1-2\eta_0c_{\beta}} \lVert \nabla_{\theta}\mathcal{L}^{\beta}(\theta_{t-1}) \rVert_2
    \le (1-\eta_0c_{\beta}) \lVert \nabla_{\theta}\mathcal{L}^{\beta}(\theta_{t-1}) \rVert_2
    \le KR_0(1-\eta_0c_{\beta})^t.
\end{equation}
Hence, $\lVert \theta_{t+1} - \theta_t \rVert_2 = \eta_0 \lVert \nabla_{\theta}\mathcal{L}^{\beta}(\theta_t) \rVert_2 \le \eta_0 KR_0(1-\eta_0c_{\beta})^t $. From this, we can follow that
\begin{equation}
    \lVert \theta_{t+1} - \theta_0 \rVert_2
    \le \sum_{u=0}^t \lVert \theta_{u+1} - \theta_u \rVert_2
    \le \eta_0 K R_0 \sum_{u=0}^t (1-\eta_0 c_{\beta})^u
    = \eta_0 K R_0 \frac{1- (1-\eta_0 c_{\beta})^{t+1}}{\eta_0 c_{\beta}}< \frac{KR_0}{c_{\beta}} = C.
\end{equation}
This proves the first two inequalities. The rest follows directly from the local Lipschitzness of the Jacobian, like in the proof for the gradient flow.
\end{proof}

\subsection{Closeness to the linearized network along the regularized gradient descent}
The following Theorem reads the same as in the gradient flow case, and the proof is very similar, which is why we only provide the main idea.
\begin{theorem}
Let $\beta \ge 0$. Let $\delta_0 > 0$ be arbitrarily small. Then, there are $C_1, C_2>0$, such that for $n$ large enough, with probability of at least $1-\delta_0$ over random initialization,
\begin{equation}
    \sup_{t\ge 0} \lVert \theta_t^{\lin} - \theta_t \rVert_2 \le C_1 \frac{(\log n)^c}{\sqrt{n}}, \quad
    \forall \lVert x\Vert_2 \le 1: \sup_{t\ge 0} \lVert f^{\lin}(x,\theta^{\lin}_t) - f(x,\theta_t) \rVert_2 \le C_2 \frac{(\log n)^c}{\sqrt{n}}.
\end{equation}
\end{theorem}
\begin{proof}[Proof Sketch]
Remember that
\begin{equation}
    f^{\lin}(x,\theta) = f(x,\theta_0) + J(x,\theta_0)(\theta - \theta_0), \text{ and } \theta^{\lin}_{t+1} = \theta^{\lin}_{t} - \eta_0 \left( J(\theta_0)^{\top}g^{\lin}(\theta^{\lin}_t) + \beta (\theta^{\lin}_t - \theta_0)\right).
\end{equation}
The structure of the proof is the same as for the gradient flow.
We will only show how to bound the term $\lVert f(x,\theta_t) - f^{\lin}(x,\theta_t) \rVert_2$. The bounds for the other terms can be done similarly. In particular, we will show by induction that
\begin{equation}
    \lVert f(x,\theta_t) - f^{\lin}(x,\theta_t) \rVert_2
    \le \eta_0 \frac{(\log n)^c}{\sqrt{n}} K'C KR_0\sum_{u=0}^{t-1}\left( 1-\eta_0c_{\beta} \right)^{u}.
\end{equation}
For $t=0$, this is true. Now, assume this holds for $s\le t$, then
\begin{align}
    &\lVert f(x,\theta_{t+1}) - f^{\lin}(x,\theta_{t+1}) \rVert_2 \\
    \le& \lVert f(x,\theta_{t}) - f^{\lin}(x,\theta_{t}) \rVert_2 
    + \lVert f(x,\theta_{t+1}) - f(x,\theta_{t}) - \left(f^{\lin}(x,\theta_{t+1})  - f^{\lin}(x,\theta_{t}) \right)\rVert_2.
\end{align}
By the chain rule and the fundamental theorem of calculus, we can write
\begin{align}
    & \lVert f(x,\theta_{t+1}) - f(x,\theta_{t}) - \left(f^{\lin}(x,\theta_{t+1})  - f^{\lin}(x,\theta_{t}) \right)\rVert_2 \\
    =& \norm{\int_0^{\eta_0} J\left(x,\theta_{t} - u \nabla_{\theta}\mathcal{L}^{\beta}(\theta_{t})\right) \nabla_{\theta}\mathcal{L}^{\beta}(\theta_{t}) du - \int_0^{\eta_0} J(x,\theta_0) \nabla_{\theta} \mathcal{L}^{\beta}(\theta_{t})du}_2 \\
    \le & \int_0^{\eta_0} \lVert J\left(x,\theta_{t} - u \nabla_{\theta}\mathcal{L}^{\beta}(\theta_{t})\right) - J(x,\theta_0) \rVert_2 \lVert  \nabla_{\theta} \mathcal{L}^{\beta}(\theta_{t}) \rVert_2 du \\
    \le & \eta_0 \frac{(\log n)^c}{\sqrt{n}}K'C KR_0\left( 1-\eta_0c_{\beta} \right)^{t}.
\end{align}
In the last step we used Theorem \ref{theorem: gradient descent geometric decay}.
Thus,
\begin{equation}
     \lVert f(x,\theta_{t+1}) - f^{\lin}(x,\theta_{t+1}) \rVert_2 \le \eta_0 \frac{(\log n)^c}{\sqrt{n}} K'C KR_0\left(\sum_{u=0}^{t-1}\left( 1-\eta_0c_{\beta} \right)^{u} + \left( 1-\eta_0c_{\beta} \right)^{t}\right).
\end{equation}
This finishes the induction proof. We can now further follow by using the geometric series that
\begin{align}
    \lVert f(x,\theta_t) - f^{\lin}(x,\theta_t) \rVert_2
    &\le \eta_0 \frac{(\log n)^c}{\sqrt{n}} K'C KR_0\sum_{u=0}^{t-1}\left( 1-\eta_0c_{\beta} \right)^{u} \\
    &< \eta_0 \frac{(\log n)^c}{\sqrt{n}} K'C KR_0\frac{1}{\eta_0 c_{\beta}} \\
    &= \frac{(\log n)^c}{\sqrt{n}} K' C^2.
\end{align}
For the other inequalities one can proceed in the same manner, using the fundamental theorem of calculus and the geometric series.
\end{proof}
\section{Shifting the network at initialization}
\label{appendix: Shifting the network at initialization}
Here, we prove that shifting the network at initialization makes it possible to include any prior mean, and compute the posterior mean with a single training run.
\shiftpredictions*
\begin{proof}
The Jacobian of the shifted network is equal to the Jacobian of the original network:
\begin{equation}
    J_{\tilde{f}}(x,\theta) = J_f(x,\theta).
\end{equation}
Define the shifted labels $\tilde{\mathbf{y}} := \mathbf{y} + f(\mathbf{x},\theta_0) - m(\mathbf{x})$. Then, $\tilde{f}(\mathbf{x}, \theta) - \mathbf{y} = f(\mathbf{x}, \theta) - \tilde{\mathbf{y}}$.
Thus, training the network $\tilde{f}$ with regularized gradient flow/descent is equivalent to training $f$ using the shifted labels $\tilde{\mathbf{y}}$, in the sense that the parameter update rule is the same. The latter leads to parameters $\theta_{\infty}$, for which (in the infinite-width limit)
\begin{align}
    f(\mathbf{x}',\theta_{\infty}) 
    &= f(\mathbf{x}', \theta_0) + \mathbf{\Theta}_{\mathbf{x',x}} \left( \mathbf{\Theta}_{\mathbf{x,x}} + \beta I \right)^{-1} (\tilde{\mathbf{y}} - f(\mathbf{x},\theta_0)).
\end{align}
By adding $-f(\mathbf{x},\theta_0) + m(\mathbf{x})$ to both sides of the equation, and using $\tilde{\mathbf{y}} - f(\mathbf{x},\theta_0) = \mathbf{y} - m(\mathbf{x})$, we get
\begin{align}
    \tilde{f}(\mathbf{x}', \theta_{\infty}) 
    = m(\mathbf{x}') + \mathbf{\Theta}_{\mathbf{x',x}} \left( \mathbf{\Theta}_{\mathbf{x,x}} + \beta I \right)^{-1} (\mathbf{y} - m(\mathbf{x})).
\end{align}
\end{proof}

\section{The Output of the Linearized Network is Gaussian over Random Initializations}
\label{appendix: gaussian-distribution-convergence}
\begin{corollary}[Convergence under Regularized Gradient Flow/Descent]\label{corollary} Under regularized gradient flow/descent training, the output of a wide neural network converges in distribution to a Gaussian over random initialization as the width $n \to \infty$. Specifically, for test inputs $\mathbf{x'}$ and $t \to \infty$, the mean and covariance of the output distribution at convergence are
\begin{align}
    \boldsymbol{\mu}(\mathbf{x'}) 
    =& \mathbf{\Theta}_{\mathbf{x', x}} \left(\mathbf{\Theta}_{\mathbf{x, x}} + \beta I\right)^{-1} \mathbf{y}, \\
    \boldsymbol{\Sigma}(\mathbf{x'}) 
    =& \mathbf{K}_{\mathbf{x', x'}} 
    + \mathbf{\Theta}_{\mathbf{x', x}} \left(\mathbf{\Theta}_{\mathbf{x, x}} + \beta I\right)^{-1} \mathbf{K}_{\mathbf{x, x}} \left(\mathbf{\Theta}_{\mathbf{x, x}} + \beta I\right)^{-1} \mathbf{\Theta}_{\mathbf{x, x'}} \\
    &- \mathbf{\Theta}_{\mathbf{x', x}} \left(\mathbf{\Theta}_{\mathbf{x, x}} + \beta I\right)^{-1} \mathbf{K}_{\mathbf{x, x'}}
    - \mathbf{K_{x',x}} \left(\mathbf{\Theta}_{\mathbf{x, x}} + \beta I\right)^{-1} \mathbf{\Theta}_{\mathbf{x, x'}}.
\end{align}
Note that the resulting covariance combines contributions from the NTK and NNGP kernels and therefore does not directly correspond to the posterior covariance of any GP.
\end{corollary}
\begin{proof}
As we showed, for large enough layer width,
\begin{equation}
    f(\mathbf{x}', \theta_{\infty}) = f(\mathbf{x}', \theta_0) + \mathbf{\Theta}_{\mathbf{x',x}}\left( \mathbf{\Theta}_{\mathbf{x,x}} + \beta I \right)^{-1} (\mathbf{y} - f(\mathbf{x}, \theta_0)).
\end{equation}
$f(\mathbf{x}', \theta_0)$ and $f(\mathbf{x}, \theta_0)$ jointly converge to a Gaussian with mean zero and covariance matrix given through the NNGP-kernel $\mathbf{K}$.
From this, it directly follows that $f(\mathbf{x'},\theta_{\infty})$ converges to a Gaussian with the given mean and covariance matrices\footnote{The covariance matrix of $X + AY$, where $X$ and $Y$ are jointly Gaussian, is given by $\Sigma_X + A\Sigma_Y A^{\top} + A\Sigma_{X,Y} + \Sigma_{Y,X}A^{\top}$.}.
\end{proof}

\end{document}